\documentclass{article}

\usepackage{microtype}
\usepackage{graphicx}
\usepackage{subfigure}
\usepackage{booktabs} % for professional tables

\usepackage{hyperref}
\usepackage{multirow}
\usepackage[accepted]{icml2024}
\usepackage{amsmath}
\usepackage{amssymb}
\usepackage{mathtools}
\usepackage{amsthm}

\usepackage[capitalize,noabbrev]{cleveref}

\theoremstyle{plain}
\newtheorem{theorem}{Theorem}[section]

\theoremstyle{definition}

\theoremstyle{remark}

\usepackage[textsize=tiny]{todonotes}

\def\cS{{\mathcal{S}}}

\def\cD{{\mathcal{D}}}
\def\cA{{\mathcal{A}}}

\def\hP{\hat{P}}
\def\hM{{\widehat{\mathcal{M}}}}
\def\hrho{{\hat{\rho}}}

\def\EE{{\mathbb{E}}}

\icmltitlerunning{Contrastive Representation for Data Filtering}

\usepackage{bbm}
\usepackage{dsfont}

\usepackage{listings}
\definecolor{codegreen}{rgb}{0,0.6,0}
\definecolor{codegray}{rgb}{0.5,0.5,0.5}
\definecolor{codepurple}{rgb}{0.58,0,0.82}
\definecolor{backcolour}{rgb}{0.95,0.95,0.92}

\lstdefinestyle{mystyle}{
    backgroundcolor=\color{backcolour},   
    commentstyle=\color{codegreen},
    keywordstyle=\color{magenta},
    numberstyle=\tiny\color{codegray},
    stringstyle=\color{codepurple},
    basicstyle=\ttfamily\footnotesize,
    breakatwhitespace=false,         
    breaklines=true,                 
    captionpos=b,                    
    keepspaces=false,                 
    numbers=left,                    
    numbersep=2pt,                  
    showspaces=false,                
    showstringspaces=false,
    showtabs=false,                  
    tabsize=2
}
\begin{document}

\twocolumn[
\icmltitle{
Contrastive Representation for Data Filtering in Cross-Domain \\ Offline Reinforcement Learning
}

% It is OKAY to include author information, even for blind
% submissions: the style file will automatically remove it for you
% unless you've provided the [accepted] option to the icml2024
% package.

% List of affiliations: The first argument should be a (short)
% identifier you will use later to specify author affiliations
% Academic affiliations should list Department, University, City, Region, Country
% Industry affiliations should list Company, City, Region, Country

% You can specify symbols, otherwise they are numbered in order.
% Ideally, you should not use this facility. Affiliations will be numbered
% in order of appearance and this is the preferred way.
\icmlsetsymbol{equal}{*}

\begin{icmlauthorlist}
\icmlauthor{Xiaoyu Wen}{nwpu}
\icmlauthor{Chenjia Bai}{pjlab,sznwpu}
\icmlauthor{Kang Xu}{fdu}
\icmlauthor{Xudong Yu}{hit}
\icmlauthor{Yang Zhang}{thu}
\icmlauthor{Xuelong Li}{telecom}
\icmlauthor{Zhen Wang}{nwpu}
%\icmlauthor{}{sch}
% \icmlauthor{Firstname8 Lastname8}{sch}
% \icmlauthor{Firstname8 Lastname8}{yyy,comp}
%\icmlauthor{}{sch}
%\icmlauthor{}{sch}
\end{icmlauthorlist}

\icmlaffiliation{nwpu}{Northwestern Polytechnical University}
\icmlaffiliation{pjlab}{Shanghai Artificial Intelligence Laboratory}
\icmlaffiliation{hit}{Harbin Institute of Technology}
\icmlaffiliation{fdu}{Fudan University}
\icmlaffiliation{thu}{Tsinghua University Shenzhen Research Institute}
\icmlaffiliation{telecom}{The Institute of Artificial Intelligence (TeleAI), China Telecom}
\icmlaffiliation{sznwpu}{Shenzhen Research Institute of Northwestern Polytechnical University}

% \icmlaffiliation{nwpu1}{Northwestern Polytechnical University, School of Artificial Intelligence, OPtics and ElectroNics (iOPEN), Xi'an, China}
% \icmlaffiliation{nwpu2}{Northwestern Polytechnical University, School of Cybersecurity, Xi'an, China}

\icmlcorrespondingauthor{Chenjia Bai}{baichenjia@pjlab.org.cn}
\icmlcorrespondingauthor{Zhen Wang}{w-zhen@nwpu.edu.cn}

% You may provide any keywords that you
% find helpful for describing your paper; these are used to populate
% the "keywords" metadata in the PDF but will not be shown in the document
\icmlkeywords{Machine Learning, ICML}

\vskip 0.3in
]

% this must go after the closing bracket ] following \twocolumn[ ...

% This command actually creates the footnote in the first column
% listing the affiliations and the copyright notice.
% The command takes one argument, which is text to display at the start of the footnote.
% The \icmlEqualContribution command is standard text for equal contribution.
% Remove it (just {}) if you do not need this facility.

\printAffiliationsAndNotice{}  % leave blank if no need to mention equal contribution
% \printAffiliationsAndNotice{\icmlEqualContribution} % otherwise use the standard text.

\begin{abstract}
Cross-domain offline reinforcement learning leverages source domain data with diverse transition dynamics to alleviate the data requirement for the target domain. However, simply merging the data of two domains leads to performance degradation due to the dynamics mismatch. Existing methods address this problem by measuring the dynamics gap via domain classifiers while relying on the assumptions of the transferability of paired domains. In this paper, we propose a novel representation-based approach to measure the domain gap, where the representation is learned through a contrastive objective by sampling transitions from different domains. We show that such an objective recovers the mutual-information gap of transition functions in two domains without suffering from the unbounded issue of the dynamics gap in handling significantly different domains. Based on the representations, we introduce a data filtering algorithm that selectively shares transitions from the source domain according to the contrastive score functions. Empirical results on various tasks demonstrate that our method achieves superior performance, using only 10\% of the target data to achieve 89.2\% of the performance on 100\% target dataset with state-of-the-art methods. 
\end{abstract}
\vspace{-2em}

\section{Introduction}
Offline Reinforcement Learning (RL) \cite{BatchRL,BCQ,offlineReview,bai2021pessimistic,yang2022rorl,bai2024pessimistic} exhibits a distinctive advantage over online RL, leveraging previously collected offline data without requiring any additional online interactions. In real-world scenarios like robotic manipulation \cite{feng2023finetuning,shi2023robust}, autonomous driving \cite{zhang2023learning}, and healthcare \cite{fatemi2022semi}, gathering a substantial offline dataset with good coverage of transitions for a specific environment is time-consuming and expensive \cite{alberti2020idda,kuo2021synthetic,bridgeV2}. 
Nevertheless, the offline RL algorithms rely heavily on the data coverage of the offline dataset \cite{zhan2022offline,deng2023false}, and the performance degenerates significantly if the amount of offline data decreases. To tackle this challenge for a specific target domain with scarce data, cross-domain offline RL leverages additional source domain data with dynamics shift to compensate for the (target) offline dataset \cite{dara,bosa}. However, as we illustrated in Figure~\ref{fig:intro motivation}(a), simply combining the dataset from source and target domains induces a significant dynamics shift due to the dynamics mismatch, leading to policy divergence and poor performance \cite{cds,uds}. Therefore, how to appropriately incorporate source domain data to improve the data efficiency for learning in the target domain remains a challenge.

There are two key problems for cross-domain offline RL: how to \emph{measure the domain gap} and how to \emph{utilize the cross-domain data}. For the first problem, prior methods directly estimate the dynamics models with offline datasets or training domain discriminators to approximate the dynamics gap. Nevertheless, the dynamics model suffers from large extrapolation errors given limited target domain data, and domain discriminators fail to provide smooth measurement for the dynamics discrepancy. For example, the dynamics gap (i.e., $\log P_{\rm source}/P_{\rm target}[s'|s,a]$) can be unbounded when the two domains mismatch significantly \cite{vgdf}. For the second problem, previous approaches modify the rewards according to the estimation of dynamics discrepancy \cite{dara} or employ pessimistic supported constraints for the source domain data \cite{bosa}. Despite these progresses, these methods typically experience rapid performance degradation when confronted with a larger dynamics gap, as shown in Figure \ref{fig:intro motivation}(b).

\begin{figure*}[t]
    \centering
    \includegraphics[width=\textwidth]{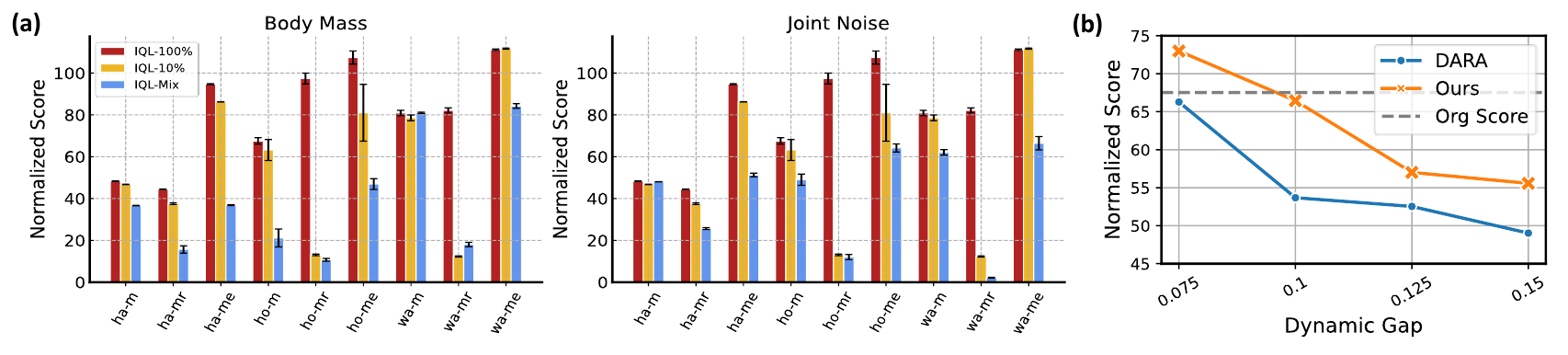}
    \vspace{-1em}
    \caption{(a) Comparison of performance across five seeds in nine Mujoco tasks (ha: halfcheetah, ho: hopper, wa: walker2d, m: medium, mr: medium-replay, me: medium-expert) with IQL \cite{iql}. We set standard D4RL \cite{d4rl} as the target domain data. For the source domain, we modify environmental parameters, such as altering body mass or introducing joint noise, and then collect offline datasets in the modified environments. (IQL-100\%: use 100\% target data, IQL-10\%: use reduced 10\% target data, IQL-Mix: use 10\% target data and 100\% source data.) (b) Comparison of performance between our algorithm and DARA \cite{dara} with 100\% source-domain dataset and 10\% target-domain dataset in the Hopper-Medium-v2 when facing the increasing dynamics gap. Specifically, we simulate a process of increasing dynamics gaps by continuously increasing the head size in the Hopper-v2 environment. The x-axis is the head size of the Hopper-v2 (normal size is 0.05), and "Org Score" is the original performance of IQL when using 100\% target data.}
    \label{fig:intro motivation}
    % \vspace{-1em}
\end{figure*}

In this paper, we propose a novel perspective to measure the domain gap via the mutual information (MI) of transitions. Specifically, we adopt the MI between the joint distribution of state-action pairs and the next states to capture the underlying dynamics of environmental transitions. Then, we use the MI gap between the source and target domains as a robust characterization of domain discrepancy when the data is shared from a significantly different source domain. In practice, such an MI gap can be estimated via a contrastive objective by using the positive samples from the target domain and the negative samples from the source domain. We employ the learned contrastive representation that captures the domain-distinguishable information as a data filter, which selectively shares the transitions from the source domain with small MI gaps to the target domain. Theoretical analysis shows that reducing the MI gap via data filtering reduces the performance bounds of two domains. Under mild assumptions, the proposed MI gap also recovers the expected dynamics gap without explicit dynamic estimators or domain discriminators.

We name the proposed method the Info-Gap Data Filtering (\textbf{IGDF}) algorithm. Empirically, we evaluate IGDF in various D4RL environments \cite{d4rl} with kinematic and morphology shifts \cite{dara,vgdf}, showcasing its superior performance compared to previous state-of-the-art algorithms. As an example of cross-domain offline RL in Figure~\ref{fig:intro motivation}(b), the MI gap used in IGDF is more robust than the dynamics gap in DARA \cite{dara}, especially for shared domains with large dynamics gaps. Our code is available in this repository (\url{https://github.com/BattleWen/IGDF}).

\section{Preliminaries}

The RL problem is typically formulated as a Markov Decision Process (MDP), defined by a tuple $\mathcal{M} = (\mathcal{S}, \mathcal{A}, P, r, \gamma, \hrho_0)$, where $\mathcal{S}$ and $\mathcal{A}$ denote the state and action spaces, $P(s'|s, a)$ is the transition dynamics, $r(s, a)$ is the reward function, $\gamma \in [0, 1)$ is the discount factor, and $\hrho_0: \mathcal{S} \rightarrow [0, 1]$ is the initial state distribution.
% \in [-R_{\text{max}}, R_{\text{max}}]

In the offline RL setting, the agent does not interact with the environment and learns a policy from an offline dataset \cite{offlineReview}. Considering a target MDP $\mathcal{M}_{\rm tar}=(\mathcal{S}, \mathcal{A}, P_{\rm tar}, r, \gamma, \hrho_0)$ has limited dataset $\mathcal{D}_{\rm tar}$. In cross-domain offline RL, we assume to access another offline dataset $\mathcal{D}_{\rm src}$ collected on a source domain MDP $\mathcal{M}_{\rm src}=(\mathcal{S}, \mathcal{A}, P_{\rm src}, r, \gamma, \hrho_0)$. We assume that all of these MDPs share the same state space, action space, and reward function and only differ in the transition probabilities, i.e., $P_{\rm src}(s'|s, a)$ and $P_{\rm tar}(s'|s, a)$. The goal of cross-domain offline RL is to leverage the additional source-domain dataset $\mathcal{D}_{\rm src}$ to relax the data requirements of the target domain. The policy is learned to maximize the expected return over the target environment $\mathcal{M}_{\rm tar}$ using the static cross-domain offline data $\mathcal{D}_{\rm mix}:= \mathcal{D}_{\rm src} \cup \mathcal{D}_{\rm tar}$.

In the offline setting, we further define the empirical MDP that estimates the expectation of the transition function $P(s'|s,a)$ from the offline dataset. Formally, an empirical MDP estimated from $\mathcal{D}$ is $\hM:=(\mathcal{S}, \mathcal{A}, \hat{P}, r, \gamma, \hrho_0)$, where $\hat{P} = \max_{\hat{P}} \mathbb{E}_{s,a,s' \sim \mathcal{D}} \left[ \log \hat{P}(s' \mid s,a) \right] $ is estimated by the maximum log-likelihood, and $\hat{P}(s' \mid s,a) = 0$ for all $(s,a,s')$ not in dataset $\mathcal{D}$. Then the empirical MDPs for the source domain and target domain are defined as $\hM_{\rm src}=(\cS,\cA,\hP_{\rm src},r,\gamma,\hrho_0)$ and $\hM_{\rm tar}=(\cS,\cA,\hP_{\rm tar},r,\gamma,\hrho_0)$, respectively. We assume the two datasets follow the same behavior policy $\pi^b(a|s)$ (refer to Appendix \ref{app:more discussions} for more details). In source MDP, $\hrho_{\rm src}(s)$ is the normalized probability that the policy $\pi^b_{\rm src}$ encounters state $s$, defined as $\hrho_{\rm src}(s)\triangleq (1-\gamma) \sum_{t=1}^{\infty}\gamma^t \hP_{\rm src}(s_t=s|\pi^b)$, and $\hrho_{\rm tar}(s)$ for the target domain follows a similar formulation.

\section{Methodology}
In this section, we first introduce the proposed MI gap, which measures the domain gap in cross-domain offline RL. Then we give a contrastive objective to estimate such a gap with learned representations. Next, we give a data filtering method to leverage the source domain data based on the representations and score functions. Finally, we give the theoretical analysis for the proposed algorithm.

\subsection{The MI Gap for Cross-Domain RL}

In the following, we denote the information measure $I(\cdot;\cdot)$ as MI and $H(\cdot)$ as Shannon entropy. We use the uppercase letter (e.g., $X$) for random variables and the lowercase letter (e.g., $x$) for their realizations. We aim to adopt the MI term to capture the dynamics-relevant information about different domains. For a distribution over the transition tuple $(s,a,s')$, we use $S,A,S'$ to stand for the corresponding random variables. We also use $p$ to denote the joint distribution of these variables as well as their associated marginals. Then the MI between the state-action pair $(S,A)$ and their future state $S'$ is defined as 
\begin{equation}
I([S,A];S')=\EE_{s,a,s'\sim \cD}\left[\log \frac{p(s,a,s')}{p(s,a) p(s')}\right],
\end{equation}
where $p(s,a)$, $p(s')$ and $p(s,a,s')$ follow empirical distributions according to the offline dataset $\cD$. For the source domain and target domain with different datasets (i.e., $\cD_{\rm src}$ and $\cD_{\rm tar}$), we denote the MI objective estimated in two domains as $I_{\rm src}([S,A];S')$ and $I_{\rm tar}([S,A];S')$, respectively. Then the MI gap between the two domains is defined as 
\begin{equation}
\Delta I = I_{\rm tar}([S,A];S') - I_{\rm src}([S,A];S'),
\end{equation}
where the two MI terms follow the different conditional and marginal probabilities. Specifically, we have
\begin{equation}
\begin{aligned}
I_{\rm tar}([S,A];S')&=\EE_{\cD}\big[\log p(s,a,s') \big/ [p(s,a) p(s')]\big]
\\&=\EE_{\cD}\big[\log \hP_{\rm tar}(s'|s,a) \big/ \hrho_{\rm tar}(s')\big],
\end{aligned}
\label{eq:mi-tar}
\end{equation}
where $\hP_{\rm tar}(s'|s,a)$ is the empirical transition function in the target-domain dataset, and $\hrho_{\rm tar}(s')$ denotes the normalized state distribution. We utilize maximum likelihood estimation in a given dataset $\mathcal{D}$ to fit $\hat{P}_{\text{tar}}(s' \,|\, s, a)$. If we denote the parameter of the empirical distribution $\hat{P}_{\text{tar}}$ by $\theta$, then the empirical distribution can be obtained by $\hat{P}_{\text{tar}}(s' \,|\, s, a) = \text{argmax}_{\theta} \sum_{(s_i,a_i,s'_i)\sim \mathcal{D}} \log \hat{P}_{\text{tar}}(s'_i| s_i, a_i ; \theta).$ The expectation in Eq.~\eqref{eq:mi-tar} follows $(s,a,s')\sim \cD$, where $\cD$ is the actual dataset for sampling transitions. For example, $\cD=\cD_{\rm tar}$ when the policy is trained with the target-domain dataset. In contrast, if the shared data from the source domain is used for training the target-domain policy, then we have $\cD=\cD_{\rm src}$. The MI term $I_{\rm src}$ of the source domain follows a similar form as Eq.~\eqref{eq:mi-tar}, but with $\hP_{\rm src}(s'|s,a)$ and $\hrho_{\rm src}(s')$ that are estimated in the source-domain dataset.

When the two domains are significantly different, the proposed MI gap $\Delta I$ is more robust than the dynamics ratio (i.e., $\Delta P=\EE_{\cD_{\rm src}}[\log \hP_{\rm tar}/ \hP_{\rm src}]$) in cross-domain data sharing. Specifically, when the samples $\{(s,a,s')\}$ from $\cD_{\rm src}$ are shared to the target domain, the probability of $\hP_{\rm tar}(s'|s,a)\rightarrow 0$ since the two domains have very different transition functions, which makes $\Delta P\rightarrow -\infty$. In contrast, the $\Delta I$ term is lower-bounded by the state entropy of behavior policies, as $\Delta I\geq -I_{\rm src}([S,A];S')\geq -H(\hrho_{\rm src}(s'))$. An illustration of the MI gap with significantly different domains is shown in Figure~\ref{fig:MI-gap}.

\begin{figure}[t]
    \centering
    \includegraphics[width=0.49\textwidth]{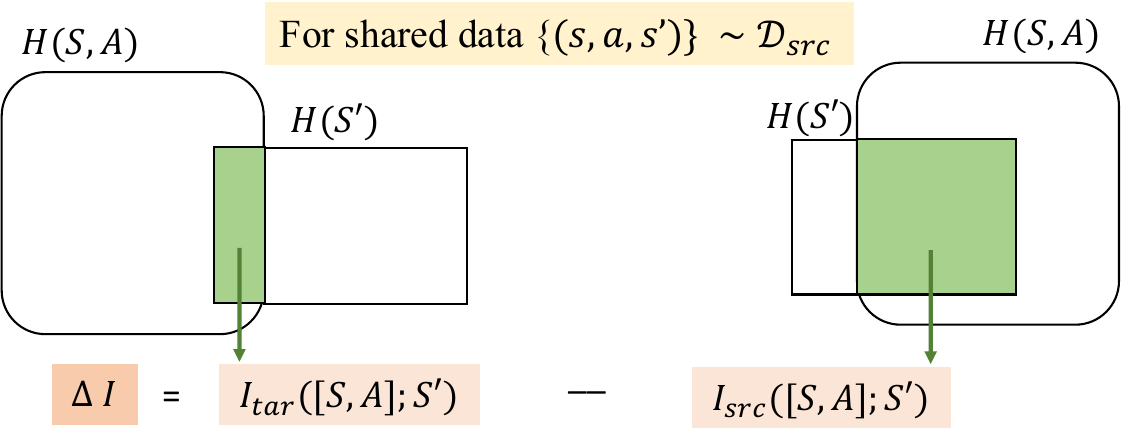}
    \vspace{-1em}
    \caption{An illustration of the MI gap of data shared from $\cD_{\rm src}$.}
    \label{fig:MI-gap}
    % \vspace{-1em}
\end{figure}

\subsection{Contrastive Representation for the MI Gap}

To estimate the MI gap in high-dimensional state space, a tractable variational estimator based on neural networks is required \cite{vabound, infoNCE, yang2023behavior}. We adopt contrastive learning to estimate the MI objective. A naive approach requires two independent estimators for $I_{\rm tar}$ and $I_{\rm src}$ separately. In contrast, we simplify this process by adopting a single contrastive objective to estimate $\Delta I$ directly. Specifically, we choose transitions $(s,a,s'_B)\sim \cD_{\rm tar}$ from the target domain as positive samples. The negative samples are obtained by first sampling a state-action pair $(s, a)\sim \cD_{\rm tar}$ and then sampling a negative state set $S'^-$ from the source domain $\cD_{\rm src}$ independently. Then a negative sample is obtained by concatenating them together to form a tuple $(s, a, s'_A)$, where $s'_A\in S'^-$. The contrastive objective can be expressed as
\begin{equation}
\label{equ:infonce}
% \mathcal{L}_{\text{NCE}} = - \mathbb{E}_{p(s,a,s'_B)} \mathbb{E}_{S'^{-}} \left[ \log\frac{h(s,a,s'_B)}{\sum_{s'_A\in S'^{-} \cup s'_B} h(s,a, s'_A)} \right],
\begin{aligned}
\mathcal{L}_{\text{NCE}} &= - \mathbb{E}_{p(s,a,s'_B)} \mathbb{E}_{S'^{-}} \\
&\left[ \log\frac{h_2(s,a,s'_B)}{h_2(s,a,s'_B)+\sum_{s'_A\in S'^{-}} h_1(s,a, s'_A)} \right],
\end{aligned}
\end{equation}
where we use two score functions to measure the information density ratio which preserves the MI between $(s,a)$ and $s'$ for the source and target domains, respectively. Intuitively, the score functions assign scores representing an exponential correlation between the state-action pair and the next state in the corresponding domains. Formally, we aim to approximate the information density of the target domain $h_2(s,a,s'_B) \propto \hP_{\rm tar}(s'_B|s,a) \big/ \hrho_{\rm tar}(s'_B)$ and source domain $h_1(s,a,s'_A) \propto \hP_{\rm src}(s'_A|s,a) \big/ \hrho_{\rm src}(s'_A)$, respectively \cite{infoNCE}.

\begin{figure*}[t]
    \centering
    \includegraphics[width=\textwidth]{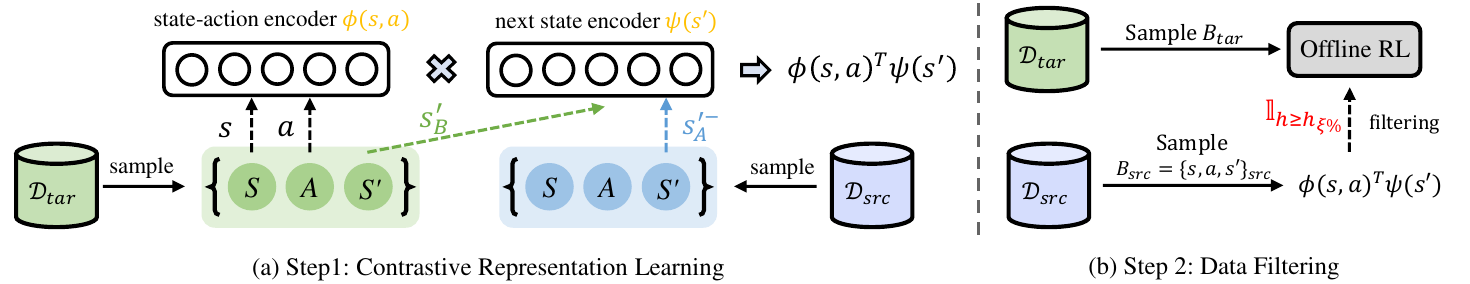}
    \vspace{-1em}
    \caption{Illustration of our method. (a) We train two encoder networks using contrastive learning, treating target transitions as positive examples and constructed transitions as negative examples. (b) We tackle cross-domain offline RL by selectively sharing the source domain data with the score functions. The target data and the share data are used for offline RL algorithms to learn the policy.}
    \label{fig:method}
    % \vspace{-1em}
\end{figure*}

The following theorem shows that the proposed contrastive objective serves as an approximate estimation for the MI gap with sufficient negative samples.
\begin{theorem}[InfoNCE extension]
\label{the:MI maximization}
The MI gap $\Delta I = I_{\rm tar}([S,A];S') - I_{\rm src}([S,A];S')$ can be lower bounded by the negative contrastive objective, as
\begin{equation}
\Delta I \geq \log(K-1) -\mathcal{L}_{\rm NCE} \coloneqq I_{\rm NCE},
\label{equ:delta_MI}
\end{equation}
where $K-1$ is the number of negative samples from the source domain.
\end{theorem}

The term $I_{\text{NCE}}$ is an asymptotically tight lower bound for the MI gap, i.e., $\lim \limits_{K\rightarrow \infty} I_{\text{NCE}}^K(X;Y) \rightarrow \Delta I(X;Y)$, which becomes tighter as $K$ becomes larger \cite{vabound, tightMI}. We refer to \S\ref{app:theoretical2} for full derivations. We illustrate the contrastive learning process in Figure~\ref{fig:method}.

For optimizing the contrastive objective in Eq.~\eqref{equ:infonce}, we adopt two score functions (i.e., $h_1$ and $h_2$) to estimate the information density of source and target domains, respectively. (\romannumeral1) $h_2(\cdot)$ only takes samples $\{(s,a,s'_B)\}$ from the target domain as inputs and assigns \textbf{high} scores to $h_2(s,a,s'_B)$, (\romannumeral2) $h_1(\cdot)$ only assigns \textbf{low} scores to $h_1(s,a,s'_A)$ for samples $\{(s,a,s'_A)\}$, where $(s,a)\in \cD_{\rm tar}$ and
$s'_A \in \cD_{\rm src}$. Interestingly, compared to training two independent contrastive estimators for $I_{\rm src}$ and $I_{\rm tar}$, our objective in Eq.~\eqref{equ:infonce} has neither negative samples for score function $h_2$, nor positive samples for score function $h_1$, which provide an opportunity to further integrate the effects of $h_1$ and $h_2$ into a single score function $h$. The new score function $h$ uses $\{(s,a,s'_B\}$ from the target domain as positive samples and constructed transitions $\{(s,a,s'_A)\}$ from two domains as negative samples. Then we have a simplified objective as
\begin{equation}
\widehat{\mathcal{L}}_{\text{NCE}} = - \mathbb{E}_{p(s,a,s'_B)} \mathbb{E}_{S'^{-}} \left[ \log\frac{h(s,a,s'_B)}{\sum_{s'_A\in S'^{-} \cup s'_B} h(s,a, s'_A)} \right],
\label{equ:infonce-simple}
\end{equation}
which serve as a simplified version of $\mathcal{L}_{\text{NCE}}$. The main reason of using a single score function is that we only share data from the source domain to the target domain (i.e., $\cD_{\rm src}\rightarrow \cD_{\rm tar}$) without a reverse data-sharing process. Thus, we do not require an independent score function to distinguish whether a transition comes from the source domain, but only required to distinguish whether a shared transition is similar to the data distribution of the target domain.

For implementation, we use two neural networks $\phi(s,a)$ and $\psi(s')$ to learn representations of state-action pairs and states only. Then we adopt a linear parameterization as
\begin{equation}
h(s,a,s') = \exp (\phi(s,a)^\top \psi(s')),
\end{equation}
which resembles spectral decomposition and low-rank representation of transition functions \cite{uehara2022representation,ren2023spectral,ren2023latent}, while we use $h(s,a,s')$ to approximate the information density. The representations are normalized to $\lVert \phi(\cdot) \rVert, \lVert \psi(\cdot) \rVert = 1$, which makes $h(\cdot)\in[1/e,e]$. In cross-domain data sharing, the score function $h$ assigns high scores for source domain data that follows a similar information density (i.e., $\hP_{\rm tar}(s'_B|s,a) \big/ \hrho_{\rm tar}(s'_B)$) to the target domain data and assigns low scores for transitions that have significantly different distributions to the target domain. 

\subsection{Data Filtering via Contrastive Representation}

Based on the representations and score function, we obtain a practical data filtering algorithm, termed IGDF (\textbf{I}nfo-\textbf{G}ap \textbf{D}ata \textbf{F}iltering), that leverages additional data with dynamics gap from the source domain to train a policy for the target MDP. Specifically, after training the encoder networks $\phi(\cdot)$ and $\psi(\cdot)$ by optimizing $\widehat{\mathcal{L}}_{\text{NCE}}$, we sample a batch of data $\{(s_A, a_A,s'_A)\}$ from $\cD_{\rm src}$ and rank the transitions according to the value of $\phi(s_A, a_A)^\top \psi(s'_A)$, then we extract the top $\xi$-quantile of batch samples for data sharing. The shared data is mixed with a batch of target domain data for offline RL training. The algorithmic description of IGDF is presented in Algorithm \ref{alg:IGDF}. In practice, data sharing is more convenient than modifying the reward function of shared data for pessimism, as it eliminates the need for meticulous adjustments to clip ranges and reward scaling ratios.

\begin{algorithm}[tb]
\caption{IGDF: Info-Gap Data Filtering algorithm}
\label{alg:IGDF}
\textbf{Input:} Source domain data $\mathcal{D}_{\rm src}$, target domain data $\mathcal{D}_{\rm tar}$
\textbf{Initialize:} Policy $\pi$, value function $Q$, encoders $\phi(s,a)$, $\psi(s')$, data filter ratio $\xi$, importance ratio $\alpha$, batch size $B$ \\
\vspace{-\baselineskip}
\begin{algorithmic}[1]
\STATE \textit{// Contrastive Representation Learning}
\STATE Optimizing the contrastive objective in Eq.~\eqref{equ:infonce-simple} to train the encoder networks $\phi(s,a)$ and $\psi(s')$
\STATE \textit{// Data Filtering algorithm}
\FOR{each gradient step}
\STATE Sample a batch $b_{\rm src}:=\{\left( s, a, r, s' \right)\}^{\frac{B}{2\xi}}$ from $\mathcal{D}_{\rm src}$
\STATE Sample a batch $b_{\rm tar}:=\{\left( s, a, r, s' \right)\}^{\frac{B}{2}}$ from $\mathcal{D}_{\rm tar}$
\STATE Sample the top-$\xi$ samples from $b_{\rm src}$ ranked by $h(\cdot)$
\STATE Combine top-$\xi$ samples in $b_{\rm src}$ and all samples in $b_{\rm tar}$
\STATE Optimize the value function $Q_\theta$ via Eq.~\eqref{equ:iw} 
\STATE Learn the policy $\pi(a|S)$ via offline RL algorithms
\ENDFOR
\end{algorithmic}
\end{algorithm}

To further enhance the performance, we propose a variant of our method by weighting the Temporal-Difference (TD)-error of filtered data using the score function. Formally, we train the value function as
\begin{align}
\label{equ:iw}
&\mathcal{L}_Q(\theta) = \frac{1}{2} \mathbb{E}_{\left(s, a, s^{\prime}\right)\sim \mathcal{D}_{\text{\rm tar}}}\left[\left(Q_{\theta}-\mathcal{T} Q_{\theta}\right)^2\right] + 
\\
&\frac{1}{2} \alpha \cdot h(s, a, s') \mathbb{E}_{\left(s, a, s^{\prime}\right)\sim \mathcal{D}_{\text{\rm src}}}\left[ \omega(s, a, s') \left(Q_{\theta}-\mathcal{T} Q_{\theta}\right)^2\right],\nonumber
\end{align}
where $\alpha$ is the importance coefficient for weighting the TD-error with the score function, and $\omega(s, a, s'):=\mathbbm{1}\left( h(s, a, s') > h_{\xi \%} \right)$ perform data filtering to extract samples with top $\xi$-quantile scores in the mini-batch sampled from the source domain. In Eq.~\eqref{equ:iw}, $\mathcal{T} Q_{\theta}$ is a general Bellman operator of the offline RL algorithms. It is also worth noting that IGDF can serve as an add-on module algorithm for arbitrary offline RL algorithms, and we select IQL as the base algorithm in experiments. The detailed procedure of IGDF+IQL is given in \S\ref{app-alg}. 
% Our code is available in this repository (\url{https://github.com/BattleWen/RO2O}).

\subsection{Theoretical Analysis}

\textbf{Connection to Dynamics Gap.} The previous methods \cite{darc,dara} for cross-domain adaptation often adopt the dynamic ratio to measure the dynamics gap. In the following, we give a connection between the dynamics gap and the proposed MI gap with transitions from different domains.
\begin{theorem}
\label{the:dg}
For shared data from the source domain $\widehat{M}_{\rm src}$, i.e., $(s,a,s')\in \cD_{\rm src}$, the relationship between the MI gap and dynamics gap is
\begin{equation}
\label{eq:deltaI-s}
\!\!\Delta I \!=\! D_{\rm KL}[\hrho_{\rm src}(s')\|\hrho_{\rm tar}(s')]-D_{\rm KL}[\hP_{\rm src}(s'|s,a)\|\hP_{\rm tar}(s'|s,a)].
\end{equation}
In contrast, for data from the target domain $\widehat{M}_{\rm tar}$, the relationship between the MI gap and dynamics gap is
\begin{equation}
\!\!\Delta I \!=\! D_{\rm KL}[\hP_{\rm tar}(s'|s,a)\|\hP_{\rm src}(s'|s,a)] - D_{\rm KL}[\hrho_{\rm tar}(s')\|\hrho_{\rm src}(s')].
\end{equation}
Then, the MI gap is bounded by 
\begin{equation}
\label{MI-bound}
-H(\hrho_{\rm src}(s'))\leq \Delta I \leq H(\hrho_{\rm tar}(s')).
\end{equation}
\end{theorem}
We give the detailed proof in \S\ref{app:thm-dg}. According to the Theorem~\ref{the:dg}, the decomposition of the MI gap also contains a KL-term to measure the dynamics gap. Nevertheless, the MI gap has an additional divergence term for state visitation distribution to regularize the dynamics gap. For example, as in Eq.~\eqref{eq:deltaI-s}, when the shared data from $\widehat{M}_{\rm src}$ is significantly different from that of $\widehat{M}_{\rm tar}$, the dynamics gap $-D_{\rm KL}[\hP_{\rm src}(s'|s,a)\|\hP_{\rm tar}(s'|s,a)]\rightarrow -\infty$, while the state density ratio $D_{\rm KL}[\hrho_{\rm src}(s')\|\hrho_{\rm tar}(s')]\rightarrow \infty$ counteracts this effect. Theoretically, we show the MI gap can be bounded by the entropy of state distribution, as in Eq.~\eqref{MI-bound}. As a result, the MI gap overcomes the drawback of the dynamics gap with large domain gaps \cite{dara} and provides a stable measurement for the domain gap.

\paragraph{Performance Guarantee.} Built on the above analysis, we provide a theoretical guarantee for sharing the source domain data from $\widehat{M}_{\rm src}$ to improve the performance of the true MDP $M_{\rm tar}$ in the target domain under the dynamic mismatch. Then we have the following performance bound for any policy $\pi$ in cross-domain offline data sharing:
\begin{theorem}
\label{thm:perf_bound}
Under the setting of corss-domain offline RL, the performance difference of any policy $\pi$ evaluated by the source domain $\widehat{M}_{\rm src}$ and the true target MDP $M_{\rm tar}$ can be bounded as below,
\begin{equation}
\label{eq:perf-bound}
\begin{aligned}
\eta_{{M}_{\rm tar}}(\pi) &\!-\! \eta_{\widehat{M}_{\rm src}}(\pi) \!\geq\! - \dfrac{\gamma R_{\rm max}}{(1-\gamma)^2}\!\Big\{
2\EE_{\hrho_{\rm tar}}\!\!\left[\!D_{\rm TV}\left(P_{\rm tar}\|\hP_{\rm tar}\right)\right] \\&
+ \sqrt{2D_{\rm KL}\left(\hrho_{\rm src}(s')\|\hrho_{\rm tar}(s')\right) + 2|\Delta I|} \Big\}.
\end{aligned}
\end{equation}
\end{theorem}
We give the detailed proof in \S\ref{app:thm3}. The first term $D_{\rm TV}\big [P_{\rm tar}\|\hP_{\rm tar}\big]$ of the divergence in Eq.~\eqref{eq:perf-bound} is caused by limited coverage of offline dataset and can be reduced by offline RL algorithms. The second divergence term includes the MI gap and the state distribution of the empirical MDP of source and target domains, which can be reduced by the proposed data filtering algorithms based on the MI gap. And we also provide additional details about a tight sub-optimality gap of IGDF in Appendix \ref{app:sub-optimaliy gap}.

\section{Related Work}
\paragraph{Dynamic adaptation in RL}
The problem of dynamic adaptation focuses on policy adaptation in domains with varying transition dynamics. Prior methods have proposed several design paradigms, including system identification \cite{sys_ident2,sys_ident3,sys_ident4,sys_ident5,sys_ident6} to capture the dynamics and visual properties of the real world, domain randomization methods \cite{domain_random1,domain_random2,domain_random3,domain_random4,domain_random5} that introduce diversity by randomly altering simulation parameters, meta-RL \cite{metarl1,metarl2} that performs fast policy fine-tuning, and imitation learning \cite{chae2022robust,kim2020domain} that learns expert policy. However, these methods require additional online interactions, offline historical transitions, or prior knowledge to select the parameters and the range of randomization. More recent works have explored the online dynamics adaptation given limited offline experiences from the target domain based on dynamics gap \cite{darc} and value function \cite{vgdf}, or via imperfect simulations from the source domain \cite{h2o,h2o+}. In contrast, we explore cross-domain adaptation in a purely offline setting based on the MI of transitions.
\vspace{-1em}
% System identification \cite{sys_ident1} aims to tune simulator system parameters \cite{sys_ident2,sys_ident3,sys_ident4,sys_ident5,sys_ident6} to capture the dynamics and visual properties of the real world. These methods generally need sufficient online interactions \cite{sys_ident_online1,sys_ident_online2} or offline historical transitions \cite{sys_ident_offline1,sys_ident_offline2} to learn a robust dynamic model. Another approach, domain randomization, introduces diversity by randomly altering simulation parameters and then trains the RL policy across multiple environments \cite{domain_random1,domain_random2,domain_random3,domain_random4,domain_random5}. While often effective, it's worth noting that this method demands prior knowledge or domain expertise including the selection of parameters and the range of randomization. Meta RL \cite{metarl1,metarl2} learns how to quickly and effectively adapt online to new tasks using only a relatively small number of training samples. In contrast, our work mainly focuses on the generalization between dynamics rather than tasks. More recent works have explored the online dynamics adaptation given limited offline experiences from the target domain \cite{vgdf} or imperfect simulations \cite{h2o,h2o+}. In contrast, we primarily explore cross-domain adaptation in a purely offline setting.
\paragraph{Cross-domain offline RL}
% Learning to act from a static dataset without any possibility of improving exploration is a well-known challenge in offline RL. Recent approaches address this issue through various strategies such as policy constraints \cite{BCQ,td3bc,BEAR,BRAC,iql,spot}, uncertainty estimations \cite{mopo,morel,uwac,edac,pbrl,rorl,ro2o,augment}, or conservative updates \cite{cql,combo}, preventing strong deviations from the state-action distribution of the dataset. However, a common assumption in prior research is that the testing environment remains consistent with the data-collecting environment. This assumption faces challenges in domains like healthcare and autonomous driving, where collecting sufficient offline data in one specific environment is often difficult and even impossible. Thus offline RL has to learn strategies from datasets with dynamics shift. 
Learning to act from a limited dataset without any possibility of improving exploration is a well-known challenge in offline RL \cite{wen2023towards,zhang2023uncertainty}. Cross-domain offline RL aims to leverage additional source domain data with dynamics shift to contribute to offline RL data efficiency. There are two preliminary problems: how to identify the dynamic discrepancy between source and target domain, and how to leverage source domain offline data. Prior works address the first problem by training two discriminators to evaluate a dynamics gap-related term \cite{darc, dara}, employing a GAN-style discriminator \cite{srpo}, or directly estimating the dynamics models \cite{bosa}. However, the learned dynamics models suffer from large extrapolation errors given limited target domain data, and domain discriminators fail to provide reliable estimation when dynamics shifts are significant \cite{vgdf}. To optimize the efficient reuse of source domain data, previous methods have explored various strategies, such as reward modification \cite{dara, srpo} and pessimistic supported constraints \cite{bosa}, but still encounter certain limitations. Their performance may degrade when confronted with a larger dynamics gap. In contrast, our method employs a representation-based approach to smoothly measure the dynamics gap, avoiding explicit estimation of transition probabilities. Moreover, we propose a score function-based data filtering method to selectively share source domain data, achieving comparable performance with fewer target domain data.

\section{Experiments}
In this section, we present empirical validations of our approach. We examine the effectiveness of our method in scenarios with various dynamics shifts. Furthermore, we provide ablation studies and qualitative analyses of our method. 

\subsection{Datasets and Baselines}
To characterize the offline dynamics shift, we consider the \textit{Halfcheetah}, \textit{Hopper}, and \textit{Walker2d} from the Gym-MuJoCo environment, using offline samples from D4RL as our target offline dataset. For the source dataset, we change the environment parameters by altering the XML file of the MuJoCo simulator following \cite{dara,vgdf} and then collect the Medium, Medium-Replay, and Medium-Expert offline datasets in the changed environments following the same data collection procedure as in D4RL (refer to Appendix \ref{app:datailed experiment setting} for the details).

We compare our algorithms with three state-of-the-art baselines in the cross-domain offline RL setting: (\romannumeral1) DARA \cite{dara} trains a pair of binary classifiers $p({\rm target}|s, a, s')$ and $p({\rm target}|s, a)$ to evaluate dynamics gap-related transition probabilities. (\romannumeral2) SRPO \cite{srpo} gives a constrained optimization formulation that regards the state distribution as a regularizer. (\romannumeral3) BOSA \cite{bosa} proposes supported policy and value optimization, which explicitly regularizes the policy and value optimization with in-support transitions. Notably, both DARA and SRPO can be flexibly applied to a wide range of offline RL algorithms, whereas BOSA stands out as a comprehensive algorithm built upon the SPOT \cite{spot} implementation. More details are given in Appendix \ref{app:implementation details}.

\begin{table*}[ht]
\centering
\caption{Results on the single-domain RL(100\% D4RL $\rightarrow$ 10\% D4RL) and cross-domain offline RL. We average our results over 5 seeds and for each seed, we compute the normalized average score using 10 episodes. And we take the results (single-domain setting with 100\% D4RL) from their original papers. (ha: halfcheetah, ho: hopper, wa: walker2d, m: medium, mr: medium-replay, me: medium-expert.)}
\label{tab:simply combining}
\resizebox{\textwidth}{!}{%
\begin{tabular}{llccccccccccc}
\toprule
                                &                               & \multicolumn{5}{c}{single-domain setting (100\% D4RL → 10\% D4RL)}                                                                                                             &  & \multicolumn{5}{c}{cross-domain setting (10\% D4RL + source data)}                                                                                                             \\ \cmidrule{3-7} \cmidrule{9-13} 
\multicolumn{1}{c}{}            & \multicolumn{1}{c}{\textbf{}} & \multicolumn{1}{c}{\textbf{BCQ}} & \multicolumn{1}{c}{\textbf{MOPO}} & \multicolumn{1}{c}{\textbf{CQL}} & \multicolumn{1}{c}{\textbf{SPOT}} & \multicolumn{1}{c}{\textbf{IQL}} &  & \multicolumn{1}{c}{\textbf{BCQ}} & \multicolumn{1}{c}{\textbf{MOPO}} & \multicolumn{1}{c}{\textbf{CQL}} & \multicolumn{1}{c}{\textbf{SPOT}} & \multicolumn{1}{c}{\textbf{IQL}} \\ \midrule
\multirow{9}{*}{body mass}  & ha-m                          & $40.7 \rightarrow 37.6$          & $42.3 \rightarrow 3.2$            & $44.4 \rightarrow 35.4$          & $58.4 \rightarrow 45.4$           & $48.3 \rightarrow 46.8$          & \vline  & 35.1                             & 6.4                               & 32.2                             & 50.3                              & 36.7                             \\
                                & ha-mr                         & $38.2 \rightarrow 1.1$           & $53.1 \rightarrow -0.1$           & $46.2 \rightarrow 0.6$           & $52.2 \rightarrow 9.8$            & $44.5 \rightarrow 37.6$          & \vline  & 40.1                             & 10.2                              & 3.3                              & 37.6                              & 15.7                             \\
                                & ha-me                         & $64.7 \rightarrow 37.3$          & $63.5 \rightarrow 4.2$            & $62.4 \rightarrow-3.3$           & $86.9 \rightarrow 46.2$           & $94.7 \rightarrow 86.2$          &  \vline & 26.4                             & 8.9                               & 12.9                             & 33.8                              & 36.8                             \\
                                & ho-m                          & $54.5 \rightarrow 37.1$          & $28 \rightarrow 4.1$              & $58 \rightarrow 43$              & $86 \rightarrow 62.5$             & $67.5 \rightarrow 63.2$          &  \vline & 25.7                             & 5                                 & 44.9                             & 85.96                             & 21.1                             \\
                                & ho-mr                         & $33.1 \rightarrow 9.3$           & $67.5 \rightarrow 1$              & $48.6 \rightarrow 9.6$           & $100.2 \rightarrow 13.7$          & $97.4 \rightarrow 13.1$          &  \vline & 28.7                             & 5.5                               & 1.4                              & 15.5                              & 10.7                             \\
                                & ho-me                         & $110.9 \rightarrow 58$           & $23.7 \rightarrow 1.6$            & $98.7 \rightarrow 59.7$          & $99.3 \rightarrow 69$             & $107.4 \rightarrow 81.1$         &  \vline & 75.4                             & 4.8                               & 53.6                             & 75.5                              & 46.9                             \\
                                & wa-m                          & $53.1 \rightarrow 32.8$          & $17.8 \rightarrow 7$              & $79.2 \rightarrow 42.9$          & $86.4 \rightarrow 65.4$           & $80.9 \rightarrow 78.6$          &  \vline & 50.9                             & 5.7                               & 80                               & 22.5                              & 81                               \\
                                & wa-mr                         & $15 \rightarrow 6.9$             & $39 \rightarrow 5.1$              & $26.7 \rightarrow 4.6$           & $91.6 \rightarrow 18.6$           & $82.2 \rightarrow 12.3$          &  \vline & 14.9                             & 3.1                               & 0.8                              & 16                                & 18                               \\
                                & wa-me                         & $57.5 \rightarrow 32.5$          & $44.6 \rightarrow 5.3$            & $111 \rightarrow 49.5$           & $112 \rightarrow 84$              & $111.2 \rightarrow 111.7$        &  \vline & 55.2                             & 5.5                               & 63.5                             & 14.3                              & 84.3                             \\ \midrule
\multirow{9}{*}{joint noise} & ha-m                          & $40.7 \rightarrow 37.6$          & $42.3 \rightarrow 3.2$            & $44.4 \rightarrow 35.4$          & $58.4 \rightarrow 45.4$           & $48.3 \rightarrow 46.8$          & \vline  & 40                               & 3.5                               & 40.7                             & 50.1                              & 48.1                             \\
                                & ha-mr                         & $38.2 \rightarrow 1.1$           & $53.1 \rightarrow -0.1$           & $46.2 \rightarrow 0.6$           & $52.2 \rightarrow 9.8$            & $44.5 \rightarrow 37.6$          &  \vline & 39.4                             & 2.6                               & 2                                & 41                                & 25.6                             \\
                                & ha-me                         & $64.7 \rightarrow 37.3$          & $63.5 \rightarrow 4.2$            & $62.4 \rightarrow-3.3$           & $86.9 \rightarrow 46.2$           & $94.7 \rightarrow 86.2$          &  \vline & 55.3                             & 1.5                               & 7.7                              & 38.1                              & 51.3                             \\
                                & ho-m                          & $54.5 \rightarrow 37.1$          & $28 \rightarrow 4.1$              & $58 \rightarrow 43$              & $86 \rightarrow 62.5$             & $67.5 \rightarrow 63.2$          &  \vline & 49                               & 9.2                               & 58                               & 41.5                              & 49                               \\
                                & ho-mr                         & $33.1 \rightarrow 9.3$           & $67.5 \rightarrow 1$              & $48.6 \rightarrow 9.6$           & $100.2 \rightarrow 13.7$          & $97.4 \rightarrow 13.1$          &  \vline & 23.8                             & 2.3                               & 2.6                              & 23                                & 12                               \\
                                & ho-me                         & $110.9 \rightarrow 58$           & $23.7 \rightarrow 1.6$            & $98.7 \rightarrow 59.7$          & $99.3 \rightarrow 69$             & $107.4 \rightarrow 81.1$         &  \vline & 96                               & 6.1                               & 73.4                             & 52                                & 64.2                             \\
                                & wa-m                          & $53.1 \rightarrow 32.8$          & $17.8 \rightarrow 7$              & $79.2 \rightarrow 42.9$          & $86.4 \rightarrow 65.4$           & $80.9 \rightarrow 78.6$          & \vline & 44.9                             & 7.8                               & 73.2                             & 38.8                              & 62                               \\
                                & wa-mr                         & $15 \rightarrow 6.9$             & $39 \rightarrow 5.1$              & $26.7 \rightarrow 4.6$           & $91.6 \rightarrow 18.6$           & $82.2 \rightarrow 12.3$          & \vline & 9.8                              & 9.3                               & 1.4                              & 10.7                              & 2.1                              \\
                                & wa-me                         & $57.5 \rightarrow 32.5$          & $44.6 \rightarrow 5.3$            & $111 \rightarrow 49.5$           & $112 \rightarrow 84$              & $111.2 \rightarrow 111.7$        & \vline & 40.6                             & 15.2                              & 109.9                            & 74.3                              & 66.5                             \\ \midrule
Average $\spadesuit$ &                               & -48.3\%                          & -88.7\%                           & -61.5\%                          & -47.1\%                           & \textbf{-25.84}\%                         & \vline  &                                  &                                   &                                  &                                   &                                  \\
Average $\clubsuit$  &                               &                                  &                                   &                                  &                                   &                                  & \vline  & -50.1\%                          & -92.5\%                           & -59.4\%                          & -50.9\%                           & -48.4\%                          \\ \bottomrule
\end{tabular}%
}
\vspace{-1em}
\end{table*}

\subsection{Motivation Example}
\textbf{Question 1.} \textit{Is simply merging cross-domain data effective for cross-domain offline RL?}

To assess the efficacy of simply merging cross-domain data, we provide the results of different methods using single-domain offline data or cross-domain data. As shown in Table \ref{tab:simply combining}, we choose some typical model-based offline RL and model-free offline RL algorithms as backbones, including BCQ \cite{BCQ}, CQL \cite{cql}, MOPO \cite{mopo}, SPOT \cite{spot}, and IQL \cite{iql}. In the single-domain setting, the numbers to the left of the arrow ($\rightarrow$) represent the scores trained on 100\% D4RL data, and the numbers to the right of that represent the scores trained on only 10\% D4RL data. In the left panel (single-domain setting), Average$\spadesuit$ represents the average performance change when the offline data is reduced (100\% $\rightarrow$ 10\%). In the right panel (cross-domain setting), Average$\clubsuit$ represents the average performance difference between the cross-domain results and the best results among baselines that are trained with 100\% D4RL. In each line, we bold the best score among baselines that are trained with 10\% D4RL data, i.e., including the single-domain 10\% D4RL setting and the cross-domain setting. 

We observe that almost all offline RL methods experience a significant performance drop when the training data size is reduced from 100\% D4RL to 10\% D4RL. Additionally, incorporating additional source-domain data (i.e., simply merging cross-domain data) may lead to poor performance compared to using only target-domain data (10\% D4RL). We argue that the primary reason is that the source offline data cannot guarantee that the same transition (state-action-next-state) can be achieved in the target environment.
\begin{table*}[ht]
\centering
\caption{Results on body mass shift and joint noise shift tasks across five seeds, where w/o Aug means training IQL with mixed datasets.}
\label{tab:mass and joints}
\resizebox{\textwidth}{!}{%
\begin{tabular}{lllll|lllll}
\toprule
mass             & w/o Aug       & DARA                   & SRPO             & \textbf{Ours}          & joint           & w/o Aug      & DARA                   & SRPO                   & \textbf{Ours}          \\ \midrule
ha-m             & 36.69 ± 0.24  & 39.37 ± 0.11           &  47.02 ± 0.09    & \textbf{47.21 ± 0.19}           & ha-m             & 48.06 ± 0.09 & \textbf{52.59 ± 1.48}  & 51.98 ± 0.31           & 50.40 ± 0.36           \\
ha-mr            & 15.68 ± 4.28  & 35.90 ± 0.90           &  36.95 ± 0.84    & \textbf{38.76 ± 0.88}           & ha-mr            & 25.62 ± 1.09 & \textbf{47.64 ± 0.44}           & 38.48 ± 1.071          & 39.11 ± 0.55           \\
ha-me            & 36.78 ± 0.38  & 51.85 ± 2.02           &  \textbf{90.71 ± 1.05}    & 89.53 ± 2.72           & ha-me            & 51.26 ± 2.23 & 83.40 ± 0.56           & 82.96 ± 1.48           & \textbf{90.93 ± 3.21}  \\
ho-m             & 21.12 ± 10.55 & 55.92 ± 4.91           &  55.89 ± 5.85    & \textbf{63.78 ± 8.43}  & ho-m             & 49.03 ± 6.67 & \textbf{55.64 ± 3.71}           & 53.50 ± 6.61           & 54.04 ± 7.89  \\
ho-mr            & 10.69 ± 1.68  & 27.71 ± 4.62           &  17.92 ± 5.86   & \textbf{27.84 ± 9.36}  & ho-mr            & 12.01 ± 3.08 & 53.04 ± 16.10          & 37.68 ± 15.34          & \textbf{63.07 ± 27.96} \\
ho-me            & 46.95 ± 6.27  & 71.43 ± 7.70           &  96.03 ± 11.89   & \textbf{96.88 ± 18.25} & ho-me            & 64.19 ± 4.76 & 67.98 ± 31.39          & 72.59 ± 34.23          & \textbf{103.97 ± 7.68}  \\
wa-m             & 81.05 ± 0.70  & \textbf{86.77 ± 0.63}  &  83.49 ± 1.13  & 83.76 ± 0.14           & wa-m             & 62.03 ± 3.29 & 73.6 ± 6.41            & 77.32 ± 4.326          & \textbf{78.76 ± 2.74}  \\
wa-mr            & 17.99 ± 2.51  & \textbf{83.89 ± 0.61}  &  77.10 ± 2.73    & 79.19 ± 1.31           & wa-mr            & 2.10 ± 0.54  & 56.52 ± 6.64  & 45.19 ± 8.68           & \textbf{58.38 ± 10.55}            \\
wa-me            & 84.28 ± 2.77  & 93.74 ± 0.25           &  108.88 ± 5.83   & \textbf{112.10 ± 0.78} & wa-me            & 66.46 ± 7.75 & 119.25 ± 1.13          & \textbf{120.33 ± 3.88} & 116.19±5.76            \\ \midrule
\textbf{Sum}     & 351.21       & 546.56                &  613.98         & \textbf{639.05}        & \textbf{Sum}     & 380.75      & 609.66                & 580.03                 & \textbf{654.85}        \\
\textbf{Average} & -51.79\%      & -23.51\%               &  -15.37\%        & \textbf{-11.89\%}      & \textbf{Average} & -45.01\%     & -14.31\%               & -19.33\%               & \textbf{-10.81\%}      \\ \bottomrule
\end{tabular}%
}
% \vspace{-1em}
\end{table*}

\subsection{Adaptation Performance in Cross-Domain RL}
\textbf{Question 2.} \textit{Can IGDF improve offline data efficiency and achieve better performance than prior methods?}

As shown in Table \ref{tab:simply combining}, across all offline RL approaches, we observe that IQL exhibits the least performance decline in the single-domain setting, with an average decrease of only 25\%, and it exhibits the highest data efficiency. For the sake of fairness, we select IQL as the common backbone for IGDF and other baselines. 
% we unify the basic offline algorithms and use IQL as the backbone for IGDF and other baselines.

To systematically investigate the adaptation performance of IGDF, we design various dynamics shift scenarios, including kinematic shifts and morphology shifts. In our main experiment, we change the body mass of agents or introduce joint noise to the motion as our source domain environment. The empirical results are presented in Tables \ref{tab:mass and joints} and \ref{tab:additional results}. We observe that IGDF achieves the highest summation of scores over 18 tasks compared to the baselines when utilizing 10\% D4RL data. When compared to the best performance of the baselines with 100\% D4RL data, IGDF exhibits the smallest performance degradation (-11.89\% and -10.81\%) among the baselines with 10\% D4RL data. As shown in Eq. \eqref{equ:delta_MI}, benefiting from a substantial number of negative samples, IGDF can obtain a precise estimation of \emph{MI gap} and consequently make significant progress in discerning whether the sampled source-domain data helps the training over the target-domain data. However, other baseline approaches suffer from limited target data, which exacerbates the under-fitting issue of domain discriminators. Therefore, IGDF outperforms other baselines and obtains SOTA results on 11 out of 18 tasks.

\begin{table*}[ht]
\centering
\caption{Results in broken thighs and morphology tasks across five seeds, where w/o Aug means training IQL with mixed datasets.}
\label{tab:broken and morph}
\resizebox{\textwidth}{!}{%
\begin{tabular}{lllll|lllllll}
\toprule
broken       & w/o Aug              & DARA               & SRPO                 & Ours                   & morph   & w/o Aug              &DARA                & SRPO               & Ours                 \\ \midrule
ha-m         & 45.70 ± 0.10           & 43.77 ± 0.31       & 45.58 ± 0.59         & \textbf{46.48 ± 0.04}  & ha-m         & 44.38 ± 0.17           &37.49 ± 1.22        & 43.06 ± 1.44       & \textbf{45.80 ± 0.43}  \\
ha-mr        & 36.08 ± 1.23           &32.99 ± 3.93        & 37.48 ± 1.33         & \textbf{38.62 ± 0.68}  & ha-mr        & 31.96 ± 5.07           &1.73 ± 1.43         & 21.05 ± 11.92      & \textbf{37.56 ± 1.60}  \\
ha-me        & 81.67 ± 6.71           &68.79 ± 12.00       & 85.06 ± 5.23         & \textbf{89.76 ± 2.36}  & ha-me        & 60.24 ± 3.58           &35.57 ± 2.54        & 53.73 ± 8.57       & \textbf{80.19 ± 8.23}  \\
ho-m         & 53.63 ± 7.96           &55.28 ± 6.87        & 58.83 ± 8.07         & \textbf{63.15 ± 7.64}  & ho-m        & 55.26 ± 4.48           &54.12 ± 6.73        & 51.31 ± 7.20       & \textbf{59.72 ± 3.64}  \\
ho-mr        & 12.62 ± 1.29           &12.06 ± 0.77        & 12.02 ± 0.61       & \textbf{13.81 ± 2.30}  & ho-mr       & 16.83 ± 4.25           &13.86 ± 3.07        & 14.46 ± 2.18       & \textbf{16.92 ± 0.62}  \\
ho-me        & 60.25 ± 37.75          &56.44 ± 31.04       & 50.71 ± 28.51      & \textbf{95.95 ± 13.33}  & ho-me       & 92.13 ± 9.74           &63.55 ± 42.19       & 102.97 ± 6.60      & \textbf{106.80 ± 5.00}   \\
wa-m         & 63.47 ± 23.82          &22.71 ± 19.57       & 78.95 ± 3.10       & \textbf{79.19 ± 5.05}  & wa-m       & 34.29 ± 10.77          &11.29 ± 10.09       & 40.64 ± 13.54      & \textbf{59.39 ± 24.87} \\
wa-mr        & 9.03 ± 6.02            &14.20 ± 3.90        & 16.27 ± 3.63       & \textbf{17.00 ± 3.52}  & wa-mr      & 8.96 ± 6.68            &1.35 ± 3.37         & 8.15 ± 5.53        & \textbf{11.46 ± 4.58}  \\ 
wa-me        & 96.33 ± 9.77           &28.90 ± 12.79       & \textbf{103.71 ± 5.20}& 94.35 ± 10.59          & wa-me      & 87.60 ± 10.02          &13.64 ± 11.13       & 95.46 ± 26.70      & \textbf{107.87 ± 3.60}  \\ \midrule
\textbf{Sum} & 458.78               &335.14              & 488.61             & \textbf{538.31}        & \textbf{Sum} & 431.65               &232.6               & 430.83             & \textbf{525.71}         \\
\textbf{Average}    & -34.83\%  & -49.38\% & -30.47\%   & \textbf{-24.63\%} & \textbf{Average}    & -39.52\%  & -66.62\%  & -41.57\%  & \textbf{-27.38\%} \\ \bottomrule
\end{tabular}%
}
% \vspace{-1em}
\end{table*}

\begin{table}[ht]
\centering
\caption{Comparison on body mass shift and joint noise shift tasks, where DARA* denotes the best score among the offline RL baselines (BCQ, CQL, IQL, and MOPO) when using dynamics-aware reward modification (DARA).}
\label{tab:additional results}
\resizebox{\columnwidth}{!}{%
\begin{tabular}{llll|llll}
\toprule
mass                          & $\text{DARA}^\ast$ & BOSA        & \textbf{Ours}   & joint                       & $\text{DARA}^\ast$              & BOSA        & \textbf{Ours} \\ \midrule
ha-m                          & 39.4               & 58.3 ± 2.5  & 47.21                            & ha-m                          & 52.6                            & 56.2 ± 0.27 & 50.40                          \\
ha-mr                         & 35.9               & 37.2 ± 0.7  & 38.76                            & ha-mr                         & 47.6                            & 51.3 ± 1.1  & 39.11                          \\
ha-me                         & 56.1               & 51.6 ± 0.1  & 89.53                            & ha-me                         & 83.4                            & 52.8 ± 0.45 & 90.93                          \\
ho-m                          & 59.3               & 82.4 ± 2.1  & 63.78                            & ho-m                          & 58.0                            & 78 ± 7.3    & 54.04                          \\
ho-mr                         & 34.1               & 39.7 ± 0.1  & 27.84                            & ho-mr                         & 53.0                            & 32.7 ± 1.3  & 63.07                          \\
ho-me                         & 99.7               & 104.2 ± 0.5 & 96.88                            & ho-me                         & 109.0                           & 96.4 ± 0.5  & 103.97                         \\
wa-m                          & 86.8               & 83 ± 2.9    & 83.76                            & wa-m                          & 81.2                            & 86.5±5.6    & 78.76                          \\
wa-mr                         & 83.9               & 21.4 ± 2    & 79.19                            & wa-mr                         & 56.5                            & 38.2 ± 4.7  & 58.38                          \\
wa-me                         & 93.3               & 86.5 ± 0.6  & 112.10                           & wa-me                         & 119.3                           & 85.8 ± 0.3  & 116.19                         \\ \midrule
\textbf{Sum} & 588.5              & 564.3       & \textbf{639.05} & \textbf{Sum} & \textbf{660.6} & 577.9       & 654.85                         \\ \bottomrule
\end{tabular}%
}
\vspace{-1em}
\end{table}
\textbf{Question 3.} \textit{Can IGDF sustain stable performance when confronted with a larger dynamics gap?}

To assess the performance of IGDF under substantial dynamics shift, we conduct tests on broken thighs and morphology tasks, following the settings outlined in \citet{vgdf}. As shown in Table \ref{tab:broken and morph}, DARA exhibits poor performance in this scenario when confronted with a larger dynamics gap. We attribute the failure of DARA to the potential unbounded issue of its estimation towards the dynamics gap based on likelihood probability. 
% With the unbounded estimation, DARA is too pessimistic to leverage most source-domain data properly. 
Similarly, SRPO can hardly show remarkable improvement on both two tasks compared to the results of training IQL with the mixed dataset. 
As the estimated dynamics gap in our IGDF (i.e., \emph{MI gap}) is bounded by Eq.\eqref{MI-bound}, it endows IGDF a consistently rational attitude to judge and utilize the whole source-domain dataset when encountered with a larger dynamics gap. As a result, IGDF can deliver a more robust performance and even achieve the SOTA results on 17 out of 18 tasks.

% We also observe that when the state space in the source domain is not significantly different from the target domain, SRPO can learn a relatively accurate value function. However, when the majority of the state space lies outside the distribution, it significantly impacts its performance.
% Similarly, SRPO can hardly show remarkable improvement on both two tasks, compared to the results by simply merging 100\% source-domain data and 10\% target-domain data. 
% In contrast, IGDF can achieve a more robust performance when facing a larger dynamics gap, since our estimation of the dynamics gap inside IGDF, bounded by Eq. \eqref{MI-bound}, can measure the dynamics gap smoothly.

% We argue that the primary reason is that estimating the dynamics gap based on likelihood probability leads to unbounded results.
% In contrast, SRPO performs well on the broken thighs task but demonstrates only moderate performance on the morphology task.
% We attribute this performance discrepancy to the utilization of the state value function as a metric. When the state space in the source domain is not significantly different from the target domain, SRPO can learn a relatively accurate value function. However, when the majority of the state space lies outside the distribution, it significantly impacts its performance. In contrast, IGDF still achieves more robust performance when facing a larger dynamics gap.

\subsection{Ablation Studies}
% To investigate the impact of design components in our method, we perform ablation analysis (refers to Appendix \ref{app:additional ablation study} for more details).
\paragraph{Data Ratio $\Gamma$.}
We employ various ratios of target domain data ($\Gamma=5\%, 10\%, 30\%, 50\%$) to investigate the algorithm's sensitivity to the amount of target domain data. The results shown in Figure \ref{fig:abl_data_ratio} demonstrate that increasing the amount of target data generally improves the performance of both IQL and IGDF, which means the data amount is a critical factor for offline performance. Consistently, IGDF achieves better performance than IQL and SRPO across a wide range of target data, showing that IGDF can fully exploit the reusable source domain transitions to enhance the training efficiency concerning the target domain. It is worth noting that with the increase in the amount of target data, the previous source data selection ratio may not be suitable for the current situation. So we can adjust the proportion of the data selection ratio appropriately to avoid significant sampling errors.
\begin{figure}[ht]
    \centering
    \includegraphics[width=0.9\columnwidth]{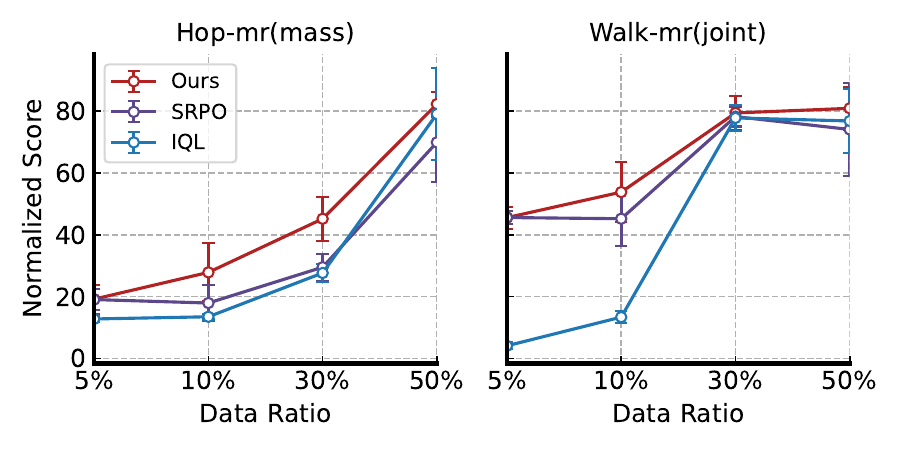}
    \vspace{-1em}
    \caption{Sensitivity on the amount of target-domain data.}
    \label{fig:abl_data_ratio}
    \vspace{-1em}
\end{figure}

\paragraph{Importance Coefficient $\alpha$}
We employ various importance coefficients ($\alpha = 0, 0.5, 1, 2$) to control the content of the importance weighting. Specifically, when alpha is equal to 0, we assign all filtered out source-domain data with the equal importance. The results shown in Figure \ref{fig:abl_impor_coef} demonstrate that incorporating importance weighting at a measured level can enhance the performance and smoothness of the algorithm. Nevertheless, striking a suitable balance presents a challenge. A too-small importance coefficient may lead to performance degradation, while an excessively large coefficient may result in unstable training.
\begin{figure}[t]
    \centering
    \includegraphics[width=\columnwidth]{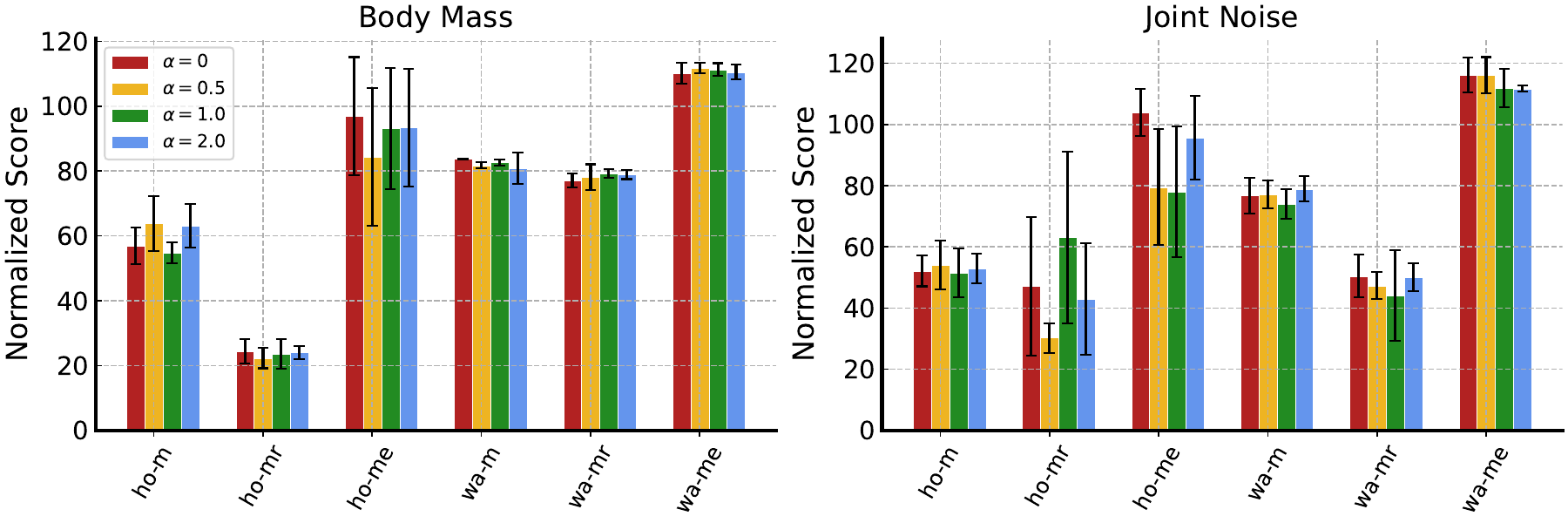}
    \vspace{-1em}
    \caption{Sensitivity on the importance coefficient.}
    \label{fig:abl_impor_coef}
    % \vspace{-2em}
\end{figure}

\section{Conclusion}
In this paper, we focus on the problem of leveraging source domain data with dynamics shifts for efficient RL training. Traditional methods often face performance degradation due to dynamics mismatch when merging cross-domain data. To address this issue, we propose a novel representation-based approach that measures the domain gap through a contrastive objective. The contrastive objective effectively captures the mutual-information gap of transition functions, providing a robust characterization of domain discrepancy without succumbing to unbounded issues. Practically, we present IGDF and the variant, serving as an add-on module for arbitrary offline RL algorithms. Empirical studies showcase the efficacy of our method outperforming previous methods, particularly in scenarios with significant dynamics gaps. One limitation of our approach is its exclusive emphasis on trajectories with similar transition probabilities while ignoring the quality of trajectories. Incorporating trajectory quality would add an interesting dimension for future exploration.

% In the unusual situation where you want a paper to appear in the
% references without citing it in the main text, use \nocite
% \nocite{langley00}

\section*{Acknowledgements}
This work is supported by the National Science Foundation for Distinguished Young Scholarship of China (No. 62025602) and the National Natural Science Foundation of China (No. 62306242).

\section*{Broader Impact}
This paper presents work whose goal is to advance the field of Machine Learning. There are many potential societal consequences of our work, none of which we feel must be specifically highlighted here.

\bibliography{igdf}
\bibliographystyle{icml2024}

%%%%%%%%%%%%%%%%%%%%%%%%%%%%%%%%%%%%%%%%%%%%%%%%%%%%%%%%%%%%%%%%%%%%%%%%%%%%%%%
%%%%%%%%%%%%%%%%%%%%%%%%%%%%%%%%%%%%%%%%%%%%%%%%%%%%%%%%%%%%%%%%%%%%%%%%%%%%%%%
% APPENDIX
%%%%%%%%%%%%%%%%%%%%%%%%%%%%%%%%%%%%%%%%%%%%%%%%%%%%%%%%%%%%%%%%%%%%%%%%%%%%%%%
%%%%%%%%%%%%%%%%%%%%%%%%%%%%%%%%%%%%%%%%%%%%%%%%%%%%%%%%%%%%%%%%%%%%%%%%%%%%%%%
\newpage
\appendix
\onecolumn
\section{Theoretical Proof}

\subsection{Proof of Theorem \ref{the:MI maximization}}
\begin{theorem}
The MI gap $\Delta I = I_{\rm tar}([S,A];S') - I_{\rm src}([S,A];S')$ can be lower bounded by the negative contrastive objective, as
\begin{equation}
\Delta I \geq \log(K-1) -\mathcal{L}_{\rm NCE} \coloneqq I_{\rm NCE},
\end{equation}
where $K-1$ is the number of negative samples from the source domain.
\end{theorem}

\begin{proof}
\label{app:theoretical2}
For the standard Info-NCE \cite{infoNCE}, we adopt the score function $h=\exp(\phi(s,a)^\top \psi(s'))$ to approximate the information density ratio of $p(s'|s,a)$ and $p(s')$, which preserves the MI between $(s,a)$ and $s'$. In Eq.~\eqref{equ:infonce}, we use $h_1$ and $h_2$ to represent the information density ratio in source and target domains, respectively. Then the contrastive objective becomes
\begin{align}
\mathcal{L}_{\text{NCE}} &= - \mathbb{E}_{p(s,a,s'_B)} \mathbb{E}_{S'^{-}} \log \left[ \frac{\frac{\hP_{\rm tar}(s'_B|s,a)}{\hrho_{\rm tar}(s'_B)}}{\frac{\hP_{\rm tar}(s'_B|s,a)}{\hrho_{\rm tar}(s'_B)}+\sum_{s'_A \in S'^{-}}\frac{\hP_{\rm src}(s'_A|s,a)}{\hrho_{\rm src}(s'_A)}} \right] \\
&= \mathbb{E}_{p(s,a,s'_B)} \mathbb{E}_{S'^{-}} \log \left[ 1 + \frac{\hrho_{\rm tar}(s'_B)}{\hP_{\rm tar}(s'_B|s,a)} \sum_{s'_A \in S'^{-}}\frac{\hP_{\rm src}(s'_A|s,a)}{\hrho_{\rm src}(s'_A)} \right] \\
&= \mathbb{E}_{p(s,a,s'_B)} \mathbb{E}_{S'^{-}} \log \left[ 1 + (K-1)\frac{\hrho_{\rm tar}(s'_B)}{\hP_{\rm tar}(s'_B|s,a)} \frac{1}{K-1} \sum_{s'_A \in S'^{-}}\frac{\hP_{\rm src}(s'_A|s,a)}{\hrho_{\rm src}(s'_A)} \right] \\
&\geq \mathbb{E}_{p(s,a,s'_B)} \log \left[ (K-1)\frac{\hrho_{\rm tar}(s'_B)}{\hP_{\rm tar}(s'_B|s,a)} \frac{1}{K-1} \sum_{s'_A \in S'^{-}}\frac{\hP_{\rm src}(s'_A|s,a)}{\hrho_{\rm src}(s'_A)} \right] \\
&= \mathbb{E}_{p(s,a,s'_B)} \log \left[ \frac{1}{K-1}\sum_{s'_A \in S'^{-}} (K-1)\frac{\hrho_{\rm tar}(s'_B)}{\hP_{\rm tar}(s'_B|s,a)}\frac{\hP_{\rm src}(s'_A|s,a)}{\hrho_{\rm src}(s'_A)} \right].
\label{equ: convex}
\end{align}
Considering $\log$ is a convex function, we can derive the following from Eq.~\eqref{equ: convex} according to Jensen's inequality, as
\begin{align}
\mathcal{L}_{\text{NCE}} \nonumber
&\geq \mathbb{E}_{p(s,a,s'_B)} \left[ \frac{1}{K-1} \sum_{s'_A \in S'^{-}} \log \left[ (K-1)\frac{\hrho_{\rm tar}(s'_B)}{\hP_{\rm tar}(s'_B|s,a)}\frac{\hP_{\rm src}(s'_A|s,a)}{\hrho_{\rm src}(s'_A)} \right] \right] \\
&= \mathbb{E}_{p(s,a,s'_B)} \left[ \frac{1}{K-1} \sum_{s'_A \in S'^{-}} \left[ \log(K-1) + \log\frac{\hrho_{\rm tar}(s'_B)}{p(s'_B|s,a)} + \log\frac{\hP_{\rm src}(s'_A|s,a)}{\hrho_{\rm src}(s'_A)} \right] \right] \\
&\approx \mathbb{E}_{p(s,a,s'_B)} \left[ \log(K-1) + \log\frac{\hrho_{\rm tar}(s'_B)}{\hP_{\rm tar}(s'_B|s,a)} + \mathbb{E}_{s'_A \in S'^{-}}\log\frac{\hP_{\rm src}(s'_A|s,a)}{\hrho_{\rm src}(s'_A)} \right] \\
&= -I_{\rm tar} + I_{\rm src} + \log(K-1) \\
&= -\Delta I + \log(K-1).
\end{align}
Then we have
\begin{align}
\Delta I \geq \log(K-1) - \mathcal{L}_{\text{NCE}}.
\end{align}
\end{proof}

\subsection{Proof of Theorem \ref{the:dg}}\label{app:thm-dg}

\begin{theorem}
For shared data from the source domain $\widehat{M}_{\rm src}$, i.e., $(s,a,s')\in \cD_{\rm src}$, the relationship between the MI gap and dynamics gap is
\begin{equation}
\Delta I = D_{\rm KL}[\hrho_{\rm src}(s')\|\hrho_{\rm tar}(s')]-D_{\rm KL}[\hP_{\rm src}(s'|s,a)\|\hP_{\rm tar}(s'|s,a)].
\end{equation}
In contrast, for data from the target domain $\widehat{M}_{\rm tar}$, the relationship between the MI gap and dynamics gap is
\begin{equation}
\Delta I = D_{\rm KL}[\hP_{\rm tar}(s'|s,a)\|\hP_{\rm src}(s'|s,a)] - D_{\rm KL}[\hrho_{\rm tar}(s')\|\hrho_{\rm src}(s')].
\end{equation}
Then, the MI gap is bounded by 
\begin{equation}
-H(\hrho_{\rm src}(s'))\leq \Delta I \leq H(\hrho_{\rm tar}(s')).
\end{equation}
\end{theorem}

\begin{proof}
For data shared from the source domain, we derive the MI gap as
\begin{equation}
\label{eq:deltaI-source-target}
\begin{aligned}
\Delta I &= \EE_{s,a,s'\sim \cD_{\rm src}}\left[\log \frac{\hP_{\rm tar}(s'|s,a)}{\hrho_{\rm tar}(s')}\right]-\EE_{s,a,s'\sim \cD_{\rm src}}\left[\log \frac{\hP_{\rm src}(s'|s,a)}{\hrho_{\rm src}(s')}\right]\\
&= \EE_{s,a,s'\sim \cD_{\rm src}}\left[- \log \frac{\hP_{\rm src}(s'|s,a)}{\hP_{\rm tar}(s'|s,a)} + \log \frac{\hrho_{\rm src}(s')}{\hrho_{\rm tar}(s')}\right]\\
&= -D_{\rm KL}\left[\hP_{\rm src}(s'|s,a) \big\| \hP_{\rm tar}(s'|s,a)\right] + D_{\rm KL}\left[\hrho_{\rm src}(s') \big\| \hrho_{\rm tar}(s')\right].
\end{aligned}
\end{equation}

Meanwhile, since data comes from the source domain, we have $\Delta I = I_{\rm tar}-I_{\rm src}\leq 0$ since the information density of the source domain $i(s,a,s')=\hP_{\rm src}(s'|s,a)/\hrho_{\rm src}$ is larger than that of the target domain, as we illustrated in Figure~\ref{fig:MI-gap}. Then we have 
\begin{equation}
\Delta I \geq - I_{\rm src}=-H(S')+H(S'|S,A)\geq -H(\hrho_{\rm src}(s')). 
\end{equation}

For the data comes from the target domain, we derive the MI gap as
\begin{equation}
\begin{aligned}
\Delta I &= \EE_{s,a,s'\sim \cD_{\rm tar}}\left[\log \frac{\hP_{\rm tar}(s'|s,a)}{\hrho_{\rm tar}(s')}\right]-\EE_{s,a,s'\sim \cD_{\rm tar}}\left[\log \frac{\hP_{\rm src}(s'|s,a)}{\hrho_{\rm src}(s')}\right]\\
&= \EE_{s,a,s'\sim \cD_{\rm tar}}\left[\log \frac{\hP_{\rm tar}(s'|s,a)}{\hP_{\rm src}(s'|s,a)} - \log \frac{\hrho_{\rm tar}(s')}{\hrho_{\rm src}(s')}\right]\\
&= D_{\rm KL}\left[\hP_{\rm tar}(s'|s,a) \big\| \hP_{\rm src}(s'|s,a)\right] + D_{\rm KL}\left[\hrho_{\rm tar}(s') \big\| \hrho_{\rm src}(s')\right].
\end{aligned}
\end{equation}

Similarly, since data comes from the target domain, we have $\Delta I = I_{\rm tar}-I_{\rm src}\geq 0$ since the information density of the target domain $i(s,a,s')=\hP_{\rm tar}(s'|s,a)/\hrho_{\rm tar}$ is larger than that of the source domain. Then we have 
\begin{equation}
\Delta I \leq I_{\rm tar}= H(S') - H(S'|S,A) \leq H(\hrho_{\rm src}(s')).
\end{equation}

Combing the bounds in the source domain and target domain, we have 
\begin{equation}
-H(\hrho_{\rm src}(s'))\leq \Delta I \leq H(\hrho_{\rm tar}(s')).
\end{equation}
As a result, the proposed MI gap is bounded by the entropy of state distribution. The MI gap overcomes the drawback of the dynamics gap since the dynamics gap can be unbounded with a large domain gap.

\end{proof}

\subsection{Proof of Theorem~\ref{thm:perf_bound}}
\label{app:thm3}

\begin{theorem}
\label{thm:perf_bound:app}
Under the setting of cross-domain offline RL, the performance difference of any policy $\pi$ evaluated by the source domain $\widehat{M}_{\rm src}$ and the true target MDP $M_{\rm tar}$ can be bounded as below,
\begin{equation}
\label{eq:perf-bound:app}
\eta_{{M}_{\rm tar}}(\pi) - \eta_{\widehat{M}_{\rm src}}(\pi) \geq - \dfrac{\gamma R_{\rm max}}{(1-\gamma)^2}\left\{2\EE_{\hrho_{\rm tar}}\left[D_{\rm TV}\left(P_{\rm tar}\|\hP_{\rm tar}\right)\right] + \sqrt{2D_{\rm KL}\left(\hrho_{\rm src}(s')\|\hrho_{\rm tar}(s')\right) + 2|\Delta I|}\right\}
\end{equation}
\end{theorem}

\begin{proof}
    For the performance bound $\eta_{{M}_{\rm tar}}(\pi) - \eta_{\widehat{M}_{\rm src}}(\pi)$ for any policy $\pi$, we can firstly convert the bound to the following form:
    \begin{equation}
        \eta_{{M}_{\rm tar}}(\pi) - \eta_{\widehat{M}_{\rm src}}(\pi) = \underbrace{\eta_{{M}_{\rm tar}}(\pi) - \eta_{\widehat{M}_{\rm tar}}(\pi)}_{\rm (a)} + \underbrace{\eta_{\widehat{M}_{\rm tar}}(\pi) - \eta_{\widehat{M}_{\rm src}}(\pi)}_{\rm (b)}.\label{eq:perf-bound:app:eq1}
    \end{equation}
    For term $(a)$ in the RHS, we can obtain the performance bound based on the telescoping lemma~\cite{SLBO}:
    \begin{align}
        \eta_{{M}_{\rm tar}}(\pi) - \eta_{\widehat{M}_{\rm tar}}(\pi) &= \dfrac{\gamma}{1-\gamma}\EE_{\hrho_{\rm tar}}\left[\EE_{P_{M_{\rm tar}}}\left[V^{\pi}_{M_{\rm tar}}(s')\right] - \EE_{P_{\widehat{M}_{\rm tar}}}\left[V^{\pi}_{M_{\rm tar}}(s')\right]\right] \nonumber\\
        &= \dfrac{\gamma}{1-\gamma}\EE_{\hrho_{\rm tar}}\left[\sum_{s'}(P_{\rm tar}(s'|s,a) - \hP_{\rm tar}(s'|s,a))V^{\pi}_{M_{\rm tar}}(s')\right]\\
        &\geq - \dfrac{\gamma}{1-\gamma}\EE_{\hrho_{\rm tar}}\left[\sum_{s'}\left|P_{\rm tar}(s'|s,a) - \hP_{\rm tar}(s'|s,a)\right|\dfrac{R_{\rm max}}{1-\gamma}\right]\nonumber\\
        &= - \dfrac{2\gamma R_{\rm max}}{(1-\gamma)^2}\EE_{\hrho_{\rm tar}}\left[D_{\rm TV}\left(P_{\rm tar}(s'|s,a)\|\hP_{\rm tar}(s'|s,a)\right)\right]\label{eq:perf-bound:app:bound_a}
    \end{align}
    Following a similar procedure, we can obtain the performance bound of term $(b)$ in RHS of Eq.~\ref{eq:perf-bound:app:eq1}:
    \begin{align}
        \eta_{\widehat{M}_{\rm tar}}(\pi) - \eta_{\widehat{M}_{\rm src}}(\pi) &\geq -  \dfrac{2\gamma R_{\rm max}}{(1-\gamma)^2}\EE_{\hrho_{\rm src}}\left[D_{\rm TV}\left(\hP_{\rm tar}(s'|s,a)\|\hP_{\rm src}(s'|s,a)\right)\right]\nonumber\\
        &\geq - \dfrac{2\gamma R_{\rm max}}{(1-\gamma)^2}\EE_{\hrho_{\rm src}}\left[\sqrt{\dfrac{1}{2}D_{\rm KL}\left(\hP_{\rm src}(s'|s,a)\|\hP_{\rm tar}(s'|s,a)\right)}\right],\label{eq:perf-bound:app:eq2}
    \end{align}
    where the second inequality derives from Pinsker's inequality.
    Recalling the MI gap $\Delta I$ that we aim to bound for the source domain offline dataset $D_{\rm src}$, we can build the connection between the MI gap and dynamics gap as follows:
    \begin{align*}
        \Delta I &= I_{\rm tar}\left[[S,A];S'\right] - I_{\rm src}\left[[S,A];S'\right] \\
        &= \EE_{D_{\rm src}}\left[\log \hP_{\rm tar}(s'|s,a) - \log \hrho_{\rm tar}(s')\right] - \EE_{D_{\rm src}}\left[\log \hP_{\rm src}(s'|s,a) - \log \hrho_{\rm src}(s')\right]\\
        &= \EE_{(s,a)\sim D_{\rm src}}\left[\EE_{s'\sim \hP_{\rm src}(s'|s,a)}\left[\log\dfrac{\hP_{\rm tar}(s'|s,a)}{\hP_{\rm src}(s'|s,a)}\right]\right] + \EE_{s'\sim D_{\rm src}}\left[\log\dfrac{\hrho_{\rm src}(s')}{\hrho_{\rm tar}(s')}\right]\\
        &= - \EE_{(s,a)\sim D_{\rm src}}\left[D_{\rm KL}\left(\hP_{\rm src}(s'|s,a)\|\hP_{\rm tar}(s'|s,a)\right)\right] + D_{\rm KL}\left(\hrho_{\rm src}\|\hrho_{\rm tar}\right).
    \end{align*}
    Thus, we can formulate the dynamics gap considering the empirical MDPs as:
    \begin{align*}
        \EE_{(s,a)\sim D_{\rm src}}\left[D_{\rm KL}\left(\hP_{\rm src}(s'|s,a)\|\hP_{\rm tar}(s'|s,a)\right)\right] &=  D_{\rm KL}\left(\hrho_{\rm src}\|\hrho_{\rm tar}\right) - \Delta I
    \end{align*}
    Since the empirical MDP characterizes the distribution of the offline dataset~(i.e.,~$\hat{P}(s'|s,a) = 0, \forall~(s,a,s')\notin\mathcal{D}$), the state-action distribution conditioned on any policy $\pi$ equals to that conditioned on the behavior policy $\pi^{b}$~(i.e.,~$\hrho^\pi(s,a)=\hrho^{\pi^{b}}(s,a),\forall s,a$). Thus, we can further derive Eq.~\eqref{eq:perf-bound:app:eq2} to:
    \begin{align}
        \eta_{\widehat{M}_{\rm tar}}(\pi) - \eta_{\widehat{M}_{\rm src}}(\pi) &\geq - \dfrac{2\gamma R_{\rm max}}{(1-\gamma)^2}\EE_{(s,a)\sim D_{\rm src}}\left[\sqrt{\dfrac{1}{2}D_{\rm KL}\left(\hP_{\rm src}(s'|s,a)\|\hP_{\rm tar}(s'|s,a)\right)}\right]\nonumber\\
        &\geq - \dfrac{2\gamma R_{\rm max}}{(1-\gamma)^2} \sqrt{\dfrac{1}{2}\EE_{(s,a)\sim D_{\rm src}}\left[D_{\rm KL}\left(\hP_{\rm src}(s'|s,a)\|\hP_{\rm tar}(s'|s,a)\right)\right]} \quad \text{(Jensen's inequality)}\nonumber\\
        &= - \dfrac{\sqrt{2}\gamma R_{\rm max}}{(1-\gamma)^2} \sqrt{D_{\rm KL}\left(\hrho_{\rm src}\|\hrho_{\rm tar}\right) - \Delta I}\label{eq:perf-bound:app:bound_b}
    \end{align}
    Integrating Eq.~\eqref{eq:perf-bound:app:bound_a} and Eq.~\eqref{eq:perf-bound:app:bound_b}, we can obtain the final performance bound:
    \begin{align}
        \eta_{{M}_{\rm tar}}(\pi) - \eta_{\widehat{M}_{\rm src}}(\pi) &\geq - \dfrac{2\gamma R_{\rm max}}{(1-\gamma)^2}\EE_{\hrho_{\rm tar}}\left[D_{\rm TV}\left(P_{\rm tar}(s'|s,a)\|\hP_{\rm tar}(s'|s,a)\right)\right] - \dfrac{\sqrt{2}\gamma R_{\rm max}}{(1-\gamma)^2} \sqrt{D_{\rm KL}\left(\hrho_{\rm src}\|\hrho_{\rm tar}\right) - \Delta I}\nonumber\\
        &= - \dfrac{\gamma R_{\rm max}}{(1-\gamma)^2}\left\{2\EE_{\hrho_{\rm tar}}\left[D_{\rm TV}\left(P_{\rm tar}\|\hP_{\rm tar}\right)\right] + \sqrt{2D_{\rm KL}\left(\hrho_{\rm src}(s')\|\hrho_{\rm tar}(s')\right) - 2\Delta I}\right\}
    \end{align}

Meanwhile, according to Eq.~\eqref{eq:deltaI-source-target}, since $\Delta I\leq 0$ when we share the data from the source domain to the target domain, we can rewrite the performance bound as 
\begin{equation}
\eta_{{M}_{\rm tar}}(\pi) - \eta_{\widehat{M}_{\rm src}}(\pi) \geq - \dfrac{\gamma R_{\rm max}}{(1-\gamma)^2}\left\{2\EE_{\hrho_{\rm tar}}\left[D_{\rm TV}\left(P_{\rm tar}\|\hP_{\rm tar}\right)\right] + \sqrt{2D_{\rm KL}\left(\hrho_{\rm src}(s')\|\hrho_{\rm tar}(s')\right) + 2|\Delta I|}\right\}
\end{equation}
\end{proof}

\subsection{Sub-optimality gap of IGDF}
\label{app:sub-optimaliy gap}

According to Theorem \ref{thm:perf_bound}, we have:
\begin{equation}
    \eta_{M_{\text{tar}}}(\pi) - \eta_{\widehat{M}_{\text{src}}}(\pi) \geq - \dfrac{\gamma R_{\rm max}}{(1-\gamma)^2} {2E_{\hat{\rho}_{\text{tar}}} \left[ D_{\rm TV} \left( P_{\rm tar} || \hat{P}_{\rm tar}\right)\right]+ \sqrt{2D_{\rm KL}\left(\hat{\rho}_{\rm src}(s') || \hat{\rho}_{\rm tar}(s')\right) + 2|\Delta I|} },
\end{equation}
where $2D_{\rm TV}\left(P_{\rm tar}||\hat{P}_{\rm tar}\right) = \sum_{s'} | \hat{P}_{\rm tar}(s' | s, a) - P_{\rm tar}(s' | s, a)| = || \hat{P}_{\rm tar}(s, a) - P_{\rm tar}(s, a)||_1.$

by Hoeffding’s inequality and union bound, the following inequalities hold with probability at least $1 - \delta$: 
\begin{equation}
    \max_{s, a}||\hat{P}_{\rm tar}(s, a) - P_{\rm tar}(s, a)||_1 \leq \max_{s, a} |\mathcal{S}| \cdot ||\hat{P}_{\rm tar}(s, a) - P_{\rm tar}(s, a)||_{\infty} \leq |\mathcal{S}| \sqrt{\frac{1}{2n} \ln{\frac{4|\mathcal{S} \times \mathcal{A} \times \mathcal{S}|}{\delta}}},
\end{equation}
where n is number of samples for each state action pair, $\hat{P}_{\rm tar}$ as a matrix of size $|\mathcal{S}||\mathcal{A}| \times |\mathcal{S}|$. 

Moreover, we can obtain a tighter analysis by proving an $l_1$ concentration bound for multinomial distribution directly: $\max_{s, a}||\hat{P}_{\rm tar}(s, a) - P_{\rm tar}(s, a)||_1 \leq \sqrt{\frac{1}{2n} \ln{\frac{2|\mathcal{S} \times \mathcal{A}| \cdot 2^{|\mathcal{S}|}}{\delta}}}$ (Refer to \cite{agarwal2019reinforcement} for more details). 

So we obtain the following conclusion: 
\begin{equation}
    \eta_{M_{\text{tar}}}(\pi) - \eta_{\widehat{M}_{\text{src}}}(\pi) \geq - \dfrac{\gamma R_{\rm max}}{(1-\gamma)^2} { \underbrace{E_{\hat{\rho}_{\text{tar}}} \left[ \sqrt{\frac{1}{2n} \ln{\frac{2|\mathcal{S} \times \mathcal{A}| \cdot 2^{|\mathcal{S}|}}{\delta}}} \right]}_{(a)} + \sqrt{\underbrace{2 D_{\rm KL}\left(\hat{\rho}_{\rm src}(s')||\hat{\rho}_{\rm tar}(s')\right)}_{(b)} + \underbrace{2|\Delta I|}_{(c)}}}.
\end{equation}

For term (a), sampling more target-domain data allows us to obtain a more accurate estimate of $\hat{P}_{\text{tar}}$, thereby reducing the discrepancy between the true $P_{\text{tar}}$ and the estimated $\hat{P}_{\text{tar}}$. For term (b), it is determined by the properties of the source and target domain datasets and cannot be optimized. Our focus lies on term (c), where by effectively selecting samples with smaller dynamics gaps, we can minimize $\Delta I$ and tighten the performance bound.

\section{Algorithm Description}
\label{app-alg}

The pseudocode of IGDF+IQL is presented in Algorithm \ref{alg:IGDF+IQL}. We utilize IQL \cite{iql} as our backbone.
\begin{algorithm}[ht]
\caption{Info-Gap Data Filtering algorithm based on IQL}
\label{alg:IGDF+IQL}
\textbf{Input:} Source domain offline dataset $\mathcal{D}_{\text{\rm src}}$, target domain offline dataset $\mathcal{D}_{\text{\rm tar}}$, mixed offline dataset $\mathcal{D}_{\text{mix}}$ \\
\textbf{Initialization:} Policy network $\pi_{\eta}$, value network $V_{\beta}$, $Q_{\theta}$, target $Q$ network $Q_{\hat{\theta}}$, encoder networks $\phi(s,a)$, $\psi(s')$, data selection ratio $\xi$, batch size $B$, importance coefficient $\alpha$
\begin{algorithmic}[1]
\STATE \textit{// Contrastive Representation Learning}
\STATE Maximize the mutual information by training encoder networks $\phi(s,a)$, $\psi(s')$ via Eq. \eqref{equ:infonce-simple}
\STATE \textit{// TD Learning}
\FOR{each gradient step}
    \STATE Sample $b_{\text{\rm src}}:=\{\left( s, a, r, s' \right)\}^{\frac{B}{2\xi}}_{\text{\rm src}}$ from $\mathcal{D}_{\text{\rm src}}$
    \STATE Sample $b_{\text{\rm tar}}:=\{\left( s, a, r, s' \right)\}^{\frac{B}{2}}_{\text{\rm tar}}$ from $\mathcal{D}_{\text{\rm tar}}$
    \STATE Sample the top-$\xi$ samples from $b_{\rm src}$ ranked by $h(s_{\text{\rm src}}, a_{\text{\rm src}}, s'_{\text{\rm src}}) := \exp (\phi(s_{\text{\rm src}}, a_{\text{\rm src}})^T\psi(s'_{\text{\rm src}}))$ following:
    \begin{align*}
        \omega(s, a, s') := \mathbbm{1}\left( h(s, a, s') > h_{\xi \%} \right)
    \end{align*}
    \STATE Optimize the $V_{\beta}$ function following loss:
    \begin{align*}
        L_V(\beta) = \mathbb{E}_{(s,a)\sim \mathcal{D}_{\text{mix}}} \left[ L_2^{\tau} (Q_{\hat{\theta}}(s,a) - V_{\beta}(s)) \right]
    \end{align*}
    \STATE Optimize the $Q_{\theta}$ function following loss:
    \begin{align*}
        L_Q(\theta) &= \frac{1}{2} \mathbb{E}_{(s,a,s')\sim \mathcal{D}_{\text{\rm tar}}} \left[ (r(s,a) + \gamma V_{\beta}(s')) - Q_{\theta}(s,a)\right]^2 \\
        &\quad + \frac{1}{2} \alpha \cdot h(s, a, s') \mathbb{E}_{\left(s, a, s^{\prime}\right)\sim \mathcal{D}_{\text{\rm src}}}\left[ \omega(s, a, s') \left( (r(s,a) + \gamma V_{\beta}(s')) - Q_{\theta}(s,a) \right)^2\right]
    \end{align*}
    \STATE Update the target $Q$ function: 
    \begin{align*}
        \hat{\theta} \leftarrow (1-\mu)\hat{\theta} + \mu\theta 
    \end{align*}
\ENDFOR
\STATE \textit{// Policy Extractions (AWR)}
\FOR{each gradient step}
    \STATE Optimize the policy network $\pi_{\eta}$ following loss:
    \begin{align*}
        L_{\pi}(\eta) = \mathbb{E}_{(s,a)\sim \mathcal{D}_{\text{mix}}} \left[ \exp{\lambda(Q_{\hat{\theta}} - V_{\beta}(s)) \log \pi_{\eta}(a|s)} \right]
    \end{align*}
\ENDFOR
\end{algorithmic}
\end{algorithm}

\section{Detailed Experiment Setting}
\label{app:datailed experiment setting}
\subsection{Datasets}
To generate environments with different transition functions, we design varying dynamics shift tasks based on three Mujoco benchmarks from Gym (HalfCheetah-v2, Hopper-v2, Walker2D-v2). These tasks encompass a range of modifications, such as adjusting the body mass (body mass shift), adding noises to joint (joint noise shift) of the agents, training with broken thighs and integrating morphological differences (refer to Table \ref{tab:dynamics shift setting} and Figure \ref{fig:dynmic shift tasks} for the details). For each benchmark, we categorize these tasks into two variants: \textbf{kinematic shift tasks} and \textbf{morphology shift tasks}.

\begin{figure}[ht]
    \centering
    \includegraphics[width=0.9\textwidth]{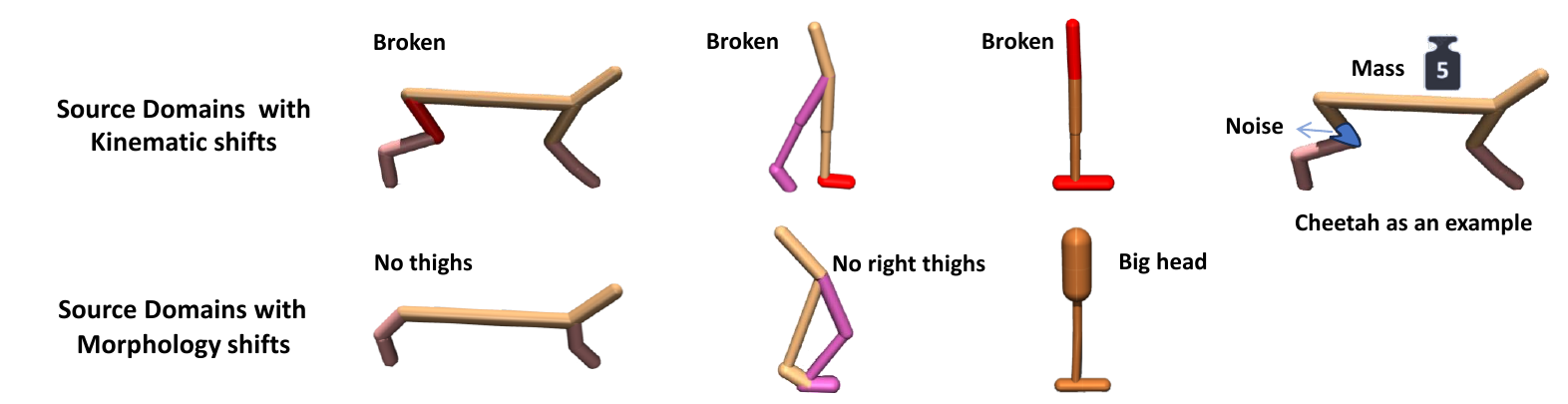}
    \caption{Illustration of all dynamics shift tasks, including kinematic shifts tasks $(\textit{Top})$ and morphology shifts tasks $(\textit{Bottom})$. For body mass shift and joint noise shift, we take halfcheetah as an example.}
    \label{fig:dynmic shift tasks}
    \vspace{-1em}
\end{figure}

As shown in Table \ref{tab:datasets statistics}, in the \textit{HalfCheetah}, \textit{Hopper}, and \textit{Walker2d} dynamics adaptation setting, we set D4RL datasets as our target domain. For the source domain, we change the environment parameters and then collect the source offline datasets in the changed environments. For body mass shift and joint noise shift, we follow the same setting of DARA, wherein 1) "Medium" offline data, generated by a trained policy with the “medium” level of performance in the source environment, 2) "Medium-Replay" offline data, consisting of recording all samples in the replay buffer observed during training until the policy reaches the “medium” level of performance, 3) "Medium-Expert" offline data, mixing equal amounts of expert demonstrations and ”medium” data in the source environment. For broken thighs and morphology shift, we alter the XML file of the Mujoco simulator following VGDF \cite{vgdf}, and then collect 1M replay transitions with SAC \cite{sac} in every benchmark. 

\begin{table}[ht]
\centering
\caption{Statistics for each task in our cross-domain offline setting.}
\label{tab:datasets statistics}
\resizebox{0.75\textwidth}{!}{%
\begin{tabular}{ccccc}
\toprule
\textbf{Environment}         & \textbf{Dynamics Shift}                                           & \textbf{Task Name} & \textbf{Target Dataset}      & \textbf{Source Dataset}        \\ \midrule
\multirow{6}{*}{HalfCheetah} & \multirow{3}{*}{Body Mass / Joint Noise}                 & Medium             & $10^5$ (D4RL) & $10^6$          \\
                             &                                                                   & Medium-Replay      & 10100 (D4RL)                 & $10^6$          \\
                             &                                                                   & Medium-Expert      & $2 \times 10^5$ (D4RL)        & $2 \times 10^6$                 \\ \cline{2-5} 
                             & \multicolumn{1}{l}{\multirow{3}{*}{Broken / Morphology}} & Medium             & $10^5$ (D4RL) & $10^6$ (Replay) \\
                             & \multicolumn{1}{l}{}                                              & Medium-Replay      & 10100 (D4RL)                 & $10^6$ (Replay) \\
                             & \multicolumn{1}{l}{}                                              & Medium-Expert      & $2 \times 10^5$ (D4RL)        & $10^6$ (Replay) \\ \midrule
\multirow{6}{*}{Hopper}      & \multirow{3}{*}{Body Mass / Joint Noise}                & Medium             & $10^5$ (D4RL) & $10^6$          \\
                             &                                                                   & Medium-Replay      & 20092 (D4RL)                 & $10^6$          \\
                             &                                                                   & Medium-Expert      & $2 \times 10^5$ (D4RL)        & $2 \times 10^6$                 \\ \cline{2-5} 
                             & \multicolumn{1}{l}{\multirow{3}{*}{Broken / Morphology}} & Medium             & $10^5$ (D4RL) & $10^6$ (Replay) \\
                             & \multicolumn{1}{l}{}                                              & Medium-Replay      & 20092 (D4RL)                 & $10^6$ (Replay) \\
                             & \multicolumn{1}{l}{}                                              & Medium-Expert      & $2 \times 10^5$ (D4RL)        & $10^6$ (Replay) \\ \midrule
\multirow{6}{*}{Walker2d}    & \multirow{3}{*}{Body Mass / Joint Noise}                & Medium             & $10^5$ (D4RL) & $10^6$          \\
                             &                                                                   & Medium-Replay      & 10093 (D4RL)                 & $10^6$          \\
                             &                                                                   & Medium-Expert      & $2 \times 10^5$ (D4RL)        & $2 \times 10^6$                 \\ \cline{2-5} 
                             & \multirow{3}{*}{Broken / Morphology}                     & Medium             & $10^5$ (D4RL) & $10^6$ (Replay) \\
                             &                                                                   & Medium-Replay      & 10093 (D4RL)                 & $10^6$ (Replay) \\
                             &                                                                   & Medium-Expert      & $2 \times 10^5$ (D4RL)        & $10^6$ (Replay) \\ \bottomrule
\end{tabular}%
}
\end{table}

\subsection{Kinematic Shift Tasks}
Detailed modifications of the environments with kinematic shifts are shown below (for changing body mass and adding joint noise, see Table \ref{tab:dynamics shift setting} for the details):
\begin{table}[!t]
% \centering
\vspace{-1em}
\caption{Dynamics shift for Halfcheetah, Hopper, Walker2d tasks. For the body mass shift, we change the mass of the body in the source MDP $\mathcal{M}_{\rm src}$. For the joint noise shift, we add a noise (randomly sampling in [-0.05, +0.05]) to the actions when we collect the source offline data.}
\label{tab:dynamics shift setting}
\resizebox{\textwidth}{!}{%
\begin{tabular}{lllllllll}
\toprule
       & \multicolumn{2}{c}{Halfcheetah}        &  & \multicolumn{2}{c}{Hopper}              &  & \multicolumn{2}{c}{Walker}               \\ \cmidrule{2-3} \cmidrule{5-6} \cmidrule{8-9} 
       & Body Mass shift & Joint noise shift    &  & Body Mass shift  & Joint noise shift    &  & Body Mass shift   & Joint noise shift    \\
\textbf{Source} & mass{[}4{]}=0.5 & action{[}-1{]}+noise &  & mass{[}-1{]}=2.5 & action{[}-1{]}+noise &  & mass{[}-1{]}=1.47 & action{[}-1{]}+noise \\
\textbf{Target} & mass{[}4{]}=1.0 & action{[}-1{]}+0     &  & mass{[}-1{]}=5.0 & action{[}-1{]}+0     &  & mass{[}-1{]}=2.94 & action{[}-1{]}+0     \\ \bottomrule
\end{tabular}%
}
\vspace{-1em}
\end{table}

\textbf{HalfCheetah - broken back thigh:} We modify the rotation range of the joint on the thigh of the back leg from $\left[- 0.52, 1.05\right]$ to $\left[- 0.0052, 0.0105\right]$.
\begin{lstlisting}[language=XML]
<joint axis="0 1 0" damping="6" name="bthigh" pos="0 0 0" range="-.0052 .0105" stiffness="240" type="hinge"/>
\end{lstlisting}

\textbf{Hopper - broken joint:} We modify the rotation range of the joint on the head from $\left[-150, 0\right]$ to $\left[-0.15, 0\right]$ and the joint on foot from $\left[-45, 45\right]$ to $\left[-18, 18\right]$.
\begin{lstlisting}[language=XML]
<joint axis="0 -1 0" name="thigh_joint" pos="0 0 1.05" range="-0.15 0" type="hinge"/>
\end{lstlisting}
\begin{lstlisting}[language=XML]
<joint axis="0 -1 0" name="foot_joint" pos="0 0 0.1" range="-18 18" type="hinge"/>
\end{lstlisting}

\textbf{Walker2d - broken right foot:} We modify the rotation range of the joint on the foot of the right leg from $\left[-45, 45\right]$ to $\left[-0.45, 0.45\right]$.
\begin{lstlisting}[language=XML]
<joint axis="0 -1 0" name="foot_joint" pos="0 0 0.1" range="-0.45 0.45" type="hinge"/>
\end{lstlisting}
\vspace{-1em}
\subsection{Morphology Shift Tasks}
Detailed modifications of the environments with morphology shifts are shown below:

\textbf{HalfCheetah - no thighs:} We modify the size of both thighs. Detailed modifications of the xml file are:
\begin{lstlisting}[language=XML]
<geom fromto="0 0 0 -0.0001 0 -0.0001" name="bthigh" size="0.046" type="capsule"/>
<body name =" bshin" pos=" -0.0001 0 -0.0001 ">
\end{lstlisting}
\begin{lstlisting}[language=XML]
<geom fromto="0 0 0 0.0001 0 0.0001" name="fthigh" size="0.046" type="capsule"/>
<body name="fshin" pos="0.0001 0 0.0001">
\end{lstlisting}

\textbf{Hopper - big head:} We modify the size of the head. Detailed modifications of the xml file are:
\begin{lstlisting}[language=XML]
<geom friction="0.9" fromto="0 0 1.45 0 0 1.05" name="torso_geom" size="0.125" type="capsule"/>
\end{lstlisting}

\textbf{Walker - no right thigh:} We modify the size of thigh on the right leg. Detailed modifications of the xml file are:
\begin{lstlisting}[language=XML]
<body name="thigh" pos="0 0 1.05">
  <joint axis="0 -1 0" name="thigh_joint" pos="0 0 1.05" range="-150 0" type="hinge"/>
  <geom friction="0.9" fromto="0 0 1.05 0 0 1.045" name ="thigh_geom" size ="0.05" type="capsule"/>
  <body name="leg" pos="0 0 0.35">
    <joint axis="0 -1 0" name="leg_joint" pos="0 0 1.045" range="-150 0" type="hinge"/>
    <geom friction="0.9" fromto="0 0 1.045 0 0 0.3" name="leg_geom" size="0.04" type="capsule"/>
    <body name="foot" pos="0.2 0 0">
        <joint axis="0 -1 0" name="foot_joint" pos="0 0 0.3" range="-45 45" type="hinge"/>
        <geom friction="0.9 " fromto="-0.0 0 0.3 0.2 0 0.3" name="foot_geom" size="0.06" type="capsule"/>
    </body>
  </body>
</body>
\end{lstlisting}

\section{Implementation Details}
\label{app:implementation details}
\subsection{Baselines}
We select DARA, SRPO, BOSA as our baselines in cross-domain offline RL tasks and choose some typical offline RL including BCQ, CQL, MOPO, IQL, SPOT as our backbones. We adopt these offline RL of open source code implemented by CORL (\href{https://github.com/tinkoff-ai/CORL}{github}). We run all algorithms with the same five random seeds. 

\textbf{DARA.} We follow the default configurations of the public implementation (\href{https://openreview.net/attachment?id=9SDQB3b68K&name=supplementary_material}{openreview}). A pair of binary classifiers $p(tar \mid s, a, s')$ and $p(tar \mid s, a)$ are learned to infer whether transitions come from the source or target domain. And the domain classifiers are trained by maximizing the cross-entropy losses:
\begin{equation}
\nonumber
\begin{gathered}
J\left(\psi_{S A S}\right) :=\mathbb{E}_{\left(s, a, s'\right) \sim D_{t a r}}\left[\log q_{\psi_{S A S}}\left(tar \mid s, a, s'\right)\right] +\mathbb{E}_{\left(s, a, s'\right) \sim D_{s r c}}\left[\log \left(1-q_{\psi_{S A S}}\left(tar \mid s, a, s'\right)\right)\right] \\
J\left(\psi_{S A}\right) :=\mathbb{E}_{\left(s, a\right) \sim D_{\text {tar}}}\left[\log q_{\psi_{S A}}\left(tar \mid s, a\right)\right]+\mathbb{E}_{\left(s, a\right) \sim D_{s r c}}\left[\log \left(1-q_{\psi_{S A}}\left(tar \mid s, a\right)\right)\right]
\end{gathered}
\end{equation}
Applying Bayes' rule, a reward correction $\Delta r\left(s, a\right)$ is augmented to the original reward $r\left(s, a\right)$ of each source domain transition during training, i.e. $\tilde{r}\left(s, a\right):=r\left(s, a\right)+\Delta r\left(s, a\right)$. The reward correction is calculated by:
$$
\Delta r\left(s, a\right):=\log \frac{\hat{P}_{\text{\rm tar}}(s' \mid s, a)}{\hat{P}_{\text{\rm src}}(s' \mid s, a)} = \log \frac{q_{\psi_{S A S}}\left(tar \mid s, a, s^{\prime}\right)}{q_{\psi_{S A S}}\left(src\mid s, a, s^{\prime}\right)} \frac{q_{\psi_{S A}}(src \mid s, a)}{q_{\psi_{S A}}(tar \mid s, a)}
$$
In practical implementation, they also clip the above reward modification between -10 and 10.

\textbf{SRPO.} We implement it based on the pseudocode and default parameters provided in the paper (\href{https://arxiv.org/pdf/2306.03552}{origin paper}). SRPO samples a batch $\mathcal{D}_{\text{batch}}$ from $\mathcal{D}_{\text{off}}$ and $\mathcal{D}_{\text{rollout}}$ and rank them by state-values. Next, SRPO trains a GAN-style discriminator to selectively choose high state-value transitions as real data and low state-value transitions as fake data. For a one-step transition $(s_{t+1}, r_t, s_t, a_t)$ in $\mathcal{D}_{\text{batch}}$, update $r_t$ with $r_t + \lambda \frac{\mathcal{D}{\delta}(s_t)}{1-\mathcal{D}_{\delta}(s_t)}$.

\textbf{BOSA.} BOSA employs two support-constrained objectives to address the out-of-distribution issues which can greatly improve offline data efficiency in cross-domain offline RL setting. Although the code is not open-source, BOSA utilizes a portion of the dataset that aligns with ours. Therefore, we directly compare our results with the scores reported in their paper(\href{https://arxiv.org/pdf/2306.12755}{origin paper}). The support optimization objectives are implemented by:
$$
\max _{\pi_\theta} \mathcal{J}_{\mathcal{D}_{\text {mix }}}\left(\pi_\theta\right):=\mathbb{E}_{\mathbf{s} \sim \mathcal{D}_{\text {mix }}, \mathbf{a} \sim \pi_\theta(\mathbf{a} \mid \mathbf{s})}\left[Q_\phi(\mathbf{s}, \mathbf{a})\right], \text { s.t. } \mathbb{E}_{\mathbf{s} \sim \mathcal{D}_{\text {mix }}}\left[\log \hat{\pi}_{\beta_{\text {mix }}}\left(\pi_\theta(\mathbf{s}) \mid \mathbf{s}\right)\right]>\epsilon_{\mathrm{th}}
$$
$$
\min _{Q_\phi} \mathcal{L}_{\text{mix}}\left(Q_\phi\right):=\mathbb{E}_{\substack{\left(\mathbf{s}, \mathbf{a}, r, \mathbf{s}^{\prime}\right) \sim \mathcal{D}_{\text{mix}} \\ \mathbf{a}^{\prime} \sim \pi_\theta\left(\mathbf{a}^{\prime} \mid \mathbf{s}^{\prime}\right)}}\left[\delta\left(Q_\phi\right) \cdot \mathbbm{1}\left(\hat{T}_{\text{target}}\left(\mathbf{s}^{\prime} \mid \mathbf{s}, \mathbf{a}\right)>\epsilon_{\text{th}}^{\prime}\right)\right]+\mathbb{E}_{(\mathbf{s}, \mathbf{a}) \sim \mathcal{D}_{\text{source}}}\left[Q_\phi(\mathbf{s}, \mathbf{a})\right]
$$

\subsection{Hyperparameters}
The hyperparameters of our backbone offline RL remain unchanged and are fixed in all tasks following the original paper. We list the basic hyperparameters of our algorithm and baselines in Table \ref{tab:hyperparameters}.
\vspace{-1em}

\begin{table*}[ht]
\centering
\caption{Hyper-parameters used for IQL, IGDF, DARA, and SRPO.}
\label{tab:hyperparameters}
\resizebox{0.7\textwidth}{!}{%
\begin{tabular}{lc}
\toprule
\textbf{IQL hyper-parameter}                & \textbf{Value}        \\ \midrule
Hidden layers (Value and Policy)                       & 2$(\text{ReLU})$      \\  
Hidden units                                & 256$(\text{MLP})$     \\
Optimizer                                   & Adam                  \\
Batch size                                  & 256                   \\
Replay buffer capacity                      & 2e6                   \\ 
Discount factor $\gamma$                    & 0.99                  \\
Target network update rate                  & 0.005                 \\
Inverse temperature  $\beta$                & 3.0                   \\
Coefficient for asymmetric loss $\tau$      & 0.7                   \\
V function learning rate                    & 3e-4                  \\
Critic learning rate                        & 3e-4                  \\
Actor learning rate                         & 3e-4                  \\ \midrule
\textbf{IGDF hyper-parameter}               & \textbf{Value}        \\ \midrule
Representation dimension $d$                & 16 or 64              \\
Contrastive encoder arch. $\phi(s,a)$           & $\text{dim}(S) + \text{dim}(A) \rightarrow 256 \rightarrow 256 \rightarrow d (\text{MLP})$ \\
Contrastive encoder arch. $\psi(s)$           & $\text{dim}(S) \rightarrow 256 \rightarrow 256 \rightarrow d (\text{MLP})$ \\
Optimizer                                   & Adam                  \\ 
Info learning rate                          & 3e-4                  \\
Info batch size                             & 128                   \\
Update number                               & 7000                  \\ 
importance coefficient $\alpha$             & 1.0                   \\ 
data selection ratio $\xi$                  & 0.25 or 0.75          \\  \bottomrule
\textbf{DARA hyper-parameter}      & \textbf{Value}                 \\ \midrule
Classifier(s,a) arch. $f(s,a)$       & $2\text{dim}(S)+ \text{dim}(A)\rightarrow 256 \rightarrow 256 \rightarrow 256 \rightarrow 2 (\text{MLP with tanh})$ \\
Classifier(s,a,s') arch. $f(s,a,s')$ & $\text{dim}(S)+ \text{dim}(A)\rightarrow 256 \rightarrow 256 \rightarrow 256 \rightarrow 2 (\text{MLP with tanh})$  \\
Optimizer                          & RMSprop                                                                                                                                                  \\
Learning rate                      & 3e-4                                                                                                                                                     \\
batch size                         & 256                                                                                                                                                      \\
Update number                      & 5000                                                                                                                                                     \\
Delta coefficient                  & 0.1                                                                                                                                                      \\ \midrule
\textbf{SRPO hyper-parameter}      & \textbf{Value}                                                                                                                                                    \\ \midrule
Hidden layers                      & 2$(\text{ReLU})$                                                                                                                                                        \\
Hidden units                       & 256$(\text{MLP})$                                                                                                                                                      \\
Optimizer                          & Adam                                                                                                                                                     \\
Learning rate                      & 3e-4                                                                                                                                                     \\
Data selection ratio               & 0.5 or 0.2                                                                                                                                                    \\
Delta coefficient $\lambda$        & 0.1 or 0.3                                                                                                                                               \\ \bottomrule
\end{tabular}%
}
\end{table*}

\newpage
\section{Supplementary Experiments}
\subsection{Ablation Study}
\label{app:additional ablation study}
\paragraph{Data Selection Ratio $\xi$.}
As the dynamics gap between source and target domains vary in different task environments, the data selection ratio becomes particularly important. We employ different data selection ratio (25\%, 50\%, 75\%, 100\%) for our algorithms. Specifically, a ratio of 100\% means that we directly learn from the mixed dataset with all source domain samples (\textit{w/o Aug}). The results shown in Figure \ref{fig:abl_data_selection} demonstrate that different tasks have varying degrees of sensitivity to dynamics gap. As expected, when we set the data selection ratio to 100\%, the performance of IGDF degrades dramatically. We observe that the \textit{Halfcheetah} and \textit{Hopper} environments are more suitable for smaller sampling ratios ($\xi = 25\%$), while the Walker2d environment is more suitable for a relatively large sampling ratio ($\xi = 75\%$). This underscores the importance of configuring the data selection ratio to achieve more robust performance when facing different dynamics gaps.

\begin{figure}[ht]
    \centering
    \includegraphics[width=0.9\textwidth]{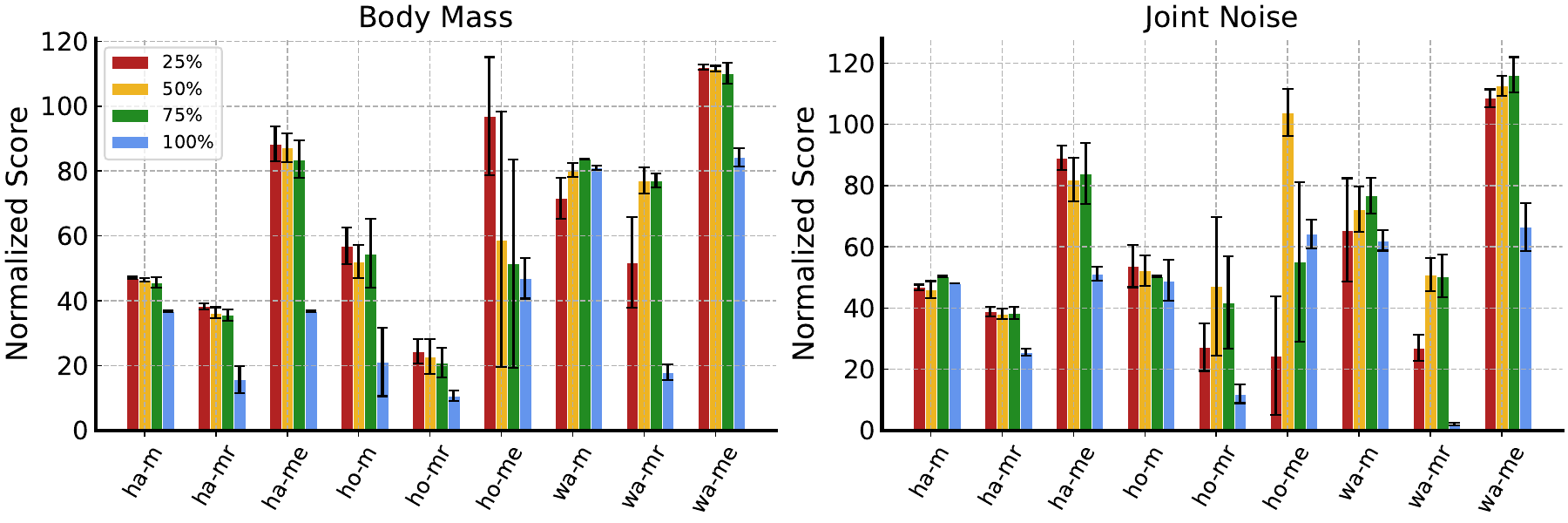}
    \vspace{-1em}
    \caption{Sensitivity on data selection ratio.}
    \label{fig:abl_data_selection}
    % \vspace{-1em}
\end{figure}

\paragraph{Representation Dimension $d$}
Equipping RL algorithms with additional representation learning components has proven effective for task solving. We employ various representation dimensions ($d = 16, 32, 64$) for encoder networks. As illustrated in Figure \ref{fig:abl_repre_dim}, we observe that the representation dimension does not have a monotonic impact on algorithm performance (a larger representation dimension does not necessarily correlate with better performance in most experiments). Moreover, larger representation dimensions even may lead to information redundancy.

\begin{figure}[ht]
    \centering
    \includegraphics[width=0.9\textwidth]{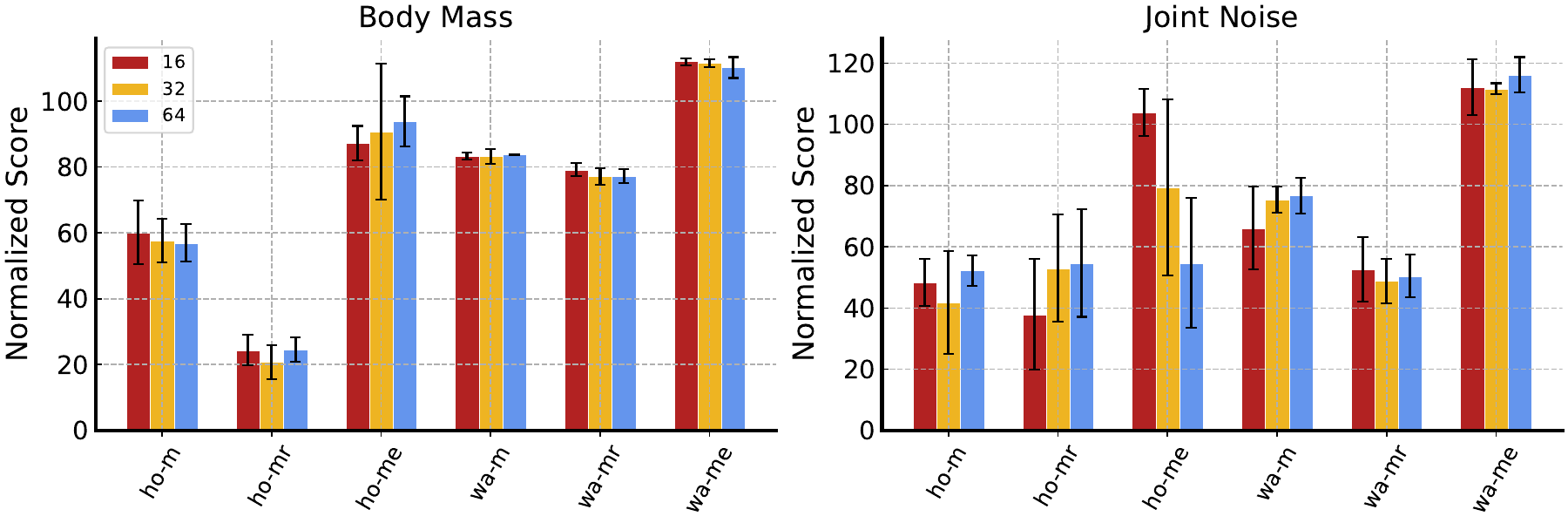}
    \vspace{-1em}
    \caption{Sensitivity on the representation dimension.}
    \label{fig:abl_repre_dim}
    \vspace{-1em}
\end{figure}

\paragraph{Weight of TD loss $h$} The weight of TD loss $h(s, a, s')$ serves as the measurement of the information density ratio, which has played an important role in further improving policy performance and training stability. It can distinguish the differences between the filtered data in a fine-grained manner by increasing the weight coefficients of samples with a smaller MI-Gap, thereby further improving learning efficiency. To assess the efficacy of using the weight $h(s, a, s')$, we perform an ablation analysis as shown in Table \ref{tab:ablation h} to evaluate the performance of IGDF without this weight.

\begin{table}[ht]
\centering
\caption{Comparative performance of using weight $h$ and without using it on body mass shift and joint noise shift tasks.}
\label{tab:ablation h}
\resizebox{0.5\textwidth}{!}{
\begin{tabular}{lll|lll}
\toprule
mass                      & w/o weight $h$ & IGDF & joint                    & w/o weight $h$ & IGDF \\ \midrule
ha-m             & 47.01±0.38     & \textbf{47.10±0.38}   & ha-m           & 46.07±2.72     & \textbf{47.84±0.76}   \\ 
ha-mr      & 38.34±0.92     & \textbf{38.76±0.88}   & ha-mr    & 38.06±1.70     & \textbf{39.11±0.55}   \\ 
ha-me      & 88.34±5.34     & \textbf{89.53±2.72}   & ha-me    & 81.97±7.10     & \textbf{90.93±3.21}   \\ 
ho-m                  & 56.90±5.72     & \textbf{63.78±8.43}   & ho-m                 & 52.17±5.08     & \textbf{54.04±7.89}   \\ 
ho-mr           & 22.74±5.32     & \textbf{27.84±9.36}   & ho-mr          & 47.07±22.75    & \textbf{63.07±27.96}  \\
ho-me           & \textbf{96.88±18.25}  & 93.82±7.63      & ho-me          & \textbf{103.97±7.68}  & 95.69±13.7       \\ 
wa-m                & \textbf{83.76±0.14}   & 82.60±1.02      & wa-m               & 76.70±5.77     & \textbf{78.76±2.74}   \\ 
wa-mr         & 77.15±2.09     & \textbf{79.19±1.31}   & wa-mr        & 50.82±5.39     & \textbf{58.38±10.55}  \\ 
wa-me         & 110.17±3.18    & \textbf{112.10±0.78}  & wa-me        & \textbf{116.19±5.76}  & 116.13±5.86      \\ \midrule
\textbf{Sum}               & 621.29         & \textbf{634.72}       & \textbf{Sum}               & 613.02         & \textbf{643.95}       \\
\textbf{Average}           & -14.32\%       & \textbf{-12.35\%}     & \textbf{Average}           & -16.55\%       & \textbf{-12.26\%}     \\ \bottomrule
\end{tabular}}
\end{table}

%%%%%%%%%%%%%%%%%%%%%%%%%%%%%%%%%%%%%%%%%%%%%%%%%%%%%%%%%%%%%%%%%%%%%%%%%%%%%%%
%%%%%%%%%%%%%%%%%%%%%%%%%%%%%%%%%%%%%%%%%%%%%%%%%%%%%%%%%%%%%%%%%%%%%%%%%%%%%%%

\subsection{Additional Experiment Results}
\paragraph{Reward modification variant} To evaluate the efficacy of the reward modification variant in our algorithm, we compare the performance of IGDF with the reward modification variant. In the reward modification approach, a reward correction term $\Delta r\left(s_t, a_t\right)$ is added to the original reward $r\left(s_t, a_t\right)$ for each source domain transition. This results in the modified reward $\tilde{r}\left(s_t, a_t\right):=r\left(s_t, a_t\right)+ \sigma \Delta r\left(s_t, a_t\right)$, where the reward correction is computed as $\phi(s_{\text{\rm src}}, a_{\text{\rm src}})^T\psi(s'_{\text{\rm src}})$. As depicted in Table \ref{tab:compare reward modification}, our observations indicate that the data filtering method exhibits significant advantages.

\begin{table}[ht]
\centering
\caption{Comparative performance of IGDF and the reward modification variant on body mass shift and joint noise shift tasks.}
\label{tab:compare reward modification}
\resizebox{\textwidth}{!}{%
\begin{tabular}{lllll|lllll}
\toprule
mass         & IGDF                   & $\sigma = 0.5$                   & $\sigma = 1.0$           & $\sigma = 2.0$           & joint        & IGDF                   & $\sigma = 0.5$                    & $\sigma = 1.0$            & $\sigma = 2.0$           \\ \midrule
ha-m         & \textbf{47.10 ± 0.38}  & 46.09 ± 1.84          & 45.95 ± 1.77  & 46.07 ± 0.67  & ha-m         & \textbf{50.40 ± 0.36}  & 49.53 ± 1.24           & 46.83 ± 0.62   & 46.68 ± 0.18  \\
ha-mr        & \textbf{38.76 ± 0.88}  & 37.99 ± 0.68          & 37.25 ± 1.24  & 36.44 ± 0.89  & ha-mr        & 39.11 ± 0.55           & \textbf{39.36 ± 0.34}  & 38.61 ± 2.00   & 37.64 ± 1.30  \\
ha-me        & \textbf{89.53 ± 2.72}  & 88.83 ± 3.24          & 86.37 ± 1.84  & 85.46 ± 6.25  & ha-me        & \textbf{90.93 ± 3.21}  & 83.86 ± 1.29           & 84.70 ± 1.201  & 82.06 ± 5.60  \\
ho-m         & \textbf{63.78 ± 8.43}  & 62.02 ± 6.56          & 59.48 ± 7.92  & 53.16 ± 4.92  & ho-m         & \textbf{54.04 ± 7.89}  & 50.63 ± 3.66           & 51.28 ± 2.41   & 48.07 ± 4.40  \\
ho-mr        & \textbf{27.84 ± 9.36}  & 24.25 ± 2.88          & 23.60 ± 6.13  & 20.58 ± 3.22  & ho-mr        & \textbf{63.07 ± 27.96} & 44.78 ± 14.78          & 41.74 ± 12.10  & 26.61 ± 5.99  \\
ho-me        & \textbf{96.88 ± 18.25} & 53.30 ± 45.00         & 78.65 ± 30.59 & 93.31 ± 18.71 & ho-me        & \textbf{103.97 ± 7.68} & 89.95 ± 23.13          & 101.64 ± 10.12 & 96.00 ± 15.55 \\
wa-m         & 83.76 ± 0.14           & \textbf{84.75 ± 0.28} & 66.78 ± 8.90  & 78.33 ± 1.60  & wa-m         & \textbf{78.76 ± 2.74}  & 78.58 ± 5.13           & 74.99 ± 13.73  & 80.79 ± 3.31  \\
wa-mr        & \textbf{79.19 ± 1.31}  & 78.20 ± 1.47          & 78.81 ± 2.58  & 78.79 ± 1.50  & wa-mr        & \textbf{58.38 ± 10.55} & 42.36 ± 9.15           & 47.80 ± 8.88   & 51.87 ± 10.01 \\
wa-me        & 112.10 ± 0.78          & 110.55 ± 2.01         & 111.93 ± 0.66 & 111.81 ± 0.74 & wa-me        & 116.19 ± 5.76          & \textbf{121.53 ± 0.02} & 119.34 ± 2.88  & 114.37 ± 2.56 \\ \midrule
\textbf{Sum} & \textbf{638.85}                 & 585.98                & 588.82        & 603.95        & \textbf{Sum} & \textbf{654.84}                 & 600.58                 & 606.912        & 584.09        \\ \bottomrule
\end{tabular}%
}
\end{table}

\paragraph{Online learning with limited target-domain data} In order to highlight the broader applicability of our work to another related line of research, we conducted additional experiments in offline-to-online settings compared with H2O \cite{h2o}. To assess the performance of H2O and IQL in online learning with limited offline data, we perform the online interactions with the source domain for $10^6$ steps and use $10^5$ target-domain transitions. For the sake of fairness, we select IQL as the backbone for IGDF and H2O. The comparison results are shown in Table~\ref{tab:compare h2o}.

\begin{table}[ht]
\centering
\caption{Comparative performance of IGDF and H2O on body mass shift and morphology shift tasks.}
\label{tab:compare h2o}
\resizebox{0.5\textwidth}{!}{%
\begin{tabular}{lll|lll}
\toprule
broken      & H2O   & IGDF       & morph & H2O   & IGDF       \\ \midrule
ha-m               & 5261 ± 76            & 5395 ± 32            & ha-m                  & 5246 ± 207          & 5351 ± 169                            \\ 
ha-mr        & 4505 ± 150           & 4469 ± 141           & ha-mr                  & 4631 ± 53           & 4512 ± 147                            \\ 
ha-me        & 8671 ± 840           & 9359 ± 553           & ha-me                  & 8807 ± 1442         & 9890 ± 874                             \\ 
ho-m                    & 1643 ± 260           & 1771 ± 339           & ho-m                  & 1642 ± 107          & 1686 ± 240                            \\ 
ho-m             & 463 ± 56             & 616 ± 257            & ho-mr                  & 417 ± 39            & 431 ± 34                               \\ 
ho-me             & 1920 ± 1057          & 2676 ± 365           & ho-m                  & 1456 ± 572          & 1773 ± 1083                       \\
wa-m                 & 3449 ± 237           & 3330 ± 528           &  wa-m                 & 3254 ± 309          & 3226 ± 538                      \\ 
wa-mr           & 404 ± 219            & 493 ± 86             &  wa-mr                 & 722 ± 182           & 630 ± 146                             \\ 
wa-me          & 4809 ± 130           & 4957 ± 99            & wa-me                  & 4247 ± 425          & 4919 ± 147                     \\ \midrule
\textbf{Sum}                        & 31125                & 33066                & \textbf{Sum}                  & 30422               & 32418 \\ \bottomrule
\end{tabular}%
}
\end{table}

\subsection{More discussions}
\label{app:more discussions}
\textbf{Question 1:} The inherent assumption of the behavior policy limits the applicability of the IGDF algorithm.

We recall the relationship between the MI gap and the dynamics gap in Equation (\ref{eq:deltaI-s}): $$\Delta I = \underbrace{D_{\text{KL}} [\hat{\rho}_{\text{src}}(s') || \hat{\rho}_{\text{tar}}(s')]}_{(a)}-\underbrace{D_{\text{KL}}[\hat{P}_{\text{src}}(s'|s,a) || \hat{P}_{\text{tar}}(s'|s,a)]}_{(b)},$$ when we use data shared from the source domain (i.e., $\mathcal{D}_{\text{src}}$) to estimate the MI gap. If the behavior policies of the two datasets are very different, the estimation of $\hat{\rho}_{\text{tar}}$ for $s'\sim \mathcal{D}_{\text{src}}$ can be difficult since the target-domain policy may never encounter similar states when interacting with the target domain, which makes $D_{\text{KL}}(\rho_{\text{tar}}(s') || \hat{\rho}_{\text{tar}}(s'))$ large, and the estimation of term (a) has a large bias. Similary, the estimation of $\hat{P}_{\text{tar}}(s'|s,a)$ for shared data $(s,a,s')\sim \mathcal{D}_{\text{src}}$ also contains large biases, which further increases the bias in estimating $\Delta I$ in data sharing.

Nevertheless, in our experiments, we find our method still achieves good results as long as there isn't a significant difference between the two behavior policies. Actually, even in data sharing between the same types of datasets (e.g., medium -> medium), the behavior policies are not entirely the same since the (medium) policies are trained in environments with dynamics gap. A more apparent evidence is shown in Table 3. In the broken and morphology tasks, we use $10^5$ D4RL transitions (medium, medium-replay, medium-expert) as our target-domain data and use $10^6$ replay transitions with SAC in every benchmark. IGDF can deliver a more robust performance and even achieve the SOTA results on 17 out of 18 tasks. We believe this assumption holds validity: if the discrepancy between the behavior policies of the two datasets is too large, the source-domain data will become useless in data sharing for the target domain.

\textbf{Question 2:} Why use linear parametrization instead of directly learning $h(s, a, s')$ in an end-to-end manner?

We choose to use linear parameterization instead of directly learning the function $h(s,a,s')$ in an end-to-end manner for several reasons: 1) Intuitively, the score function $\phi(s,a)^\top \psi(s')$ measures whether the representation of state-action pair $\phi(s,a)$ aligns with the next state $\psi(s')$. It is easier to solve a task with linear parameterization given a good representation. In our work, the representation can be separately learned via contrastive learning, which achieves better quantification. 2) As illustrated in Figure 11 of the related research \cite{darc}, solely employing the $(s,a,s')$ classifier to measure domain gaps significantly performs worse than simultaneously utilizing $(s,a,s')$ and $(s,a)$ classifiers. End-to-end learning shares a similar mechanism with solely learning the $(s,a,s')$ classifier. 3) Given what prior work has shown about RL in the presence of function approximation and state aliasing \cite{achiam2019towards,yang2022overcoming}, it is not surprising that end-to-end learning of representations is fragile \cite{laskin2020reinforcement}. RL algorithms require good representations to learn the value function and policy \cite{eysenbach2022contrastive}. 4) A recent work \cite{eysenbach2024inference} also highlighted that representations learned via InfoNCE can effectively capture conditional probabilities between random variables $x$ and $y$ (akin to the conditional probability between $(s,a)$ and $s'$ in our context).

\textbf{Question 3:} The comparison with low-rank MDPs.

Although the low-rank MDP is a theoretical-grounded assumption (i.e., $P(s'|s,a)=<\phi(s,a),\psi(s')>$) that improves the sample complexity \cite{uehara2021representation}, it can be hard to extend it to the cross-domain problem. As discussed in recent papers \cite{ren2022spectral,ren2022latent} that adopt low-rank assumption to learn representations with neural networks in high-dimensional space, the representation is learned by maximizing the likelihood as $\arg\max_{\phi,\psi}\sum\log\phi(s_i,a_i)^{\top}\psi(s')$. Then the representation $\phi(s,a)$ and $\psi(s')$ will learn to regress the transition probability in this domain. In cross-domain adaptation, if $\phi(s,a)$ and $\psi(s')$ are learned specially adapted to function $P_{\rm src}(s'|s,a)$ of the source domain, it can be hard to transfer $\phi(s,a)$ and $\psi(s')$ to the target domain since the transition probabilities of two domains are different. In contrast, the contrastive objective in our method is learned by both sampling positive sample and negative samples from both domains, which makes $h(s,a,s')=\exp(\phi(s,a)^{\top}\psi(s'))$ a score function to captures the domain-distinguishable information as a data filter. In our method, the learned representations $\phi(s,a)$ and $\psi(s')$ are not used for value/policy learning but only for data filtering.

\textbf{Question 4:} The comparison with offline multi-tasks transfer RL.

For the offline multi-task transfer problem studied in \cite{bose2024offline}, the source and target tasks are assumed to have similar transition functions to make a core assumption (i.e., Assumption 1) that all tasks share a common representation $\phi_h^\star(s,a)$ holds. However, in offline cross-domain RL considered in our paper, the representations (i.e., $\phi^\star_{\rm src}$ and $\phi^\star_{\rm tar}$) can be very different since the transition functions are very different in cross-domain settings with large domain gaps, which makes the error bound in representation transfer does not hold. Meanwhile, a pointwise linear span assumption (i.e., Assumption 2) is required in \cite{bose2024offline} to make the target transitions a linear combination of the source task dynamics. Similarly, \cite{ishfaq2024offline} also has assumptions about the shared representation $\phi^\star$, and the target task is assumed to be an $\xi$-approximated linear combination of $T$ source tasks. Nevertheless, such an assumption may not hold when facing a large dynamics gap, as we studied in our paper. Empirically, our method is robust to domain gaps and significantly outperforms other methods on 17 out of 18 tasks with large dynamics gaps (see Table 3). Another difference between our setting and \cite{bose2024offline,ishfaq2024offline} is that we do not adopt shared representation $\phi(s,a)$ for the shared domains, and the representation is only learned to capture the domain-distinguishable information as a data filter. As a result, we believe extending the theoretical results of \cite{bose2024offline,ishfaq2024offline} to cross-domain offline RL requires additional efforts to relax the assumptions to allow source and target domains to have different optimal representations. 

\end{document}